\documentclass{article}
\usepackage[nonatbib, preprint, logo]{cig}
\usepackage[utf8]{inputenc} 
\usepackage[T1]{fontenc}    
\usepackage{hyperref}
\usepackage{xcolor}
\usepackage{booktabs}
\usepackage{arydshln}
\usepackage{balance}
\usepackage{multirow}
\usepackage{multicol}
\usepackage{algorithm}
\usepackage{algpseudocode}
\usepackage{enumitem}
\definecolor{linkblue}{RGB}{0,0,128} 
\hypersetup{
    colorlinks=true,
    linkcolor=black,   %
    citecolor=linkblue,   
    urlcolor=linkblue,    
    filecolor=linkblue
}

\usepackage{wrapfig}
\usepackage{soul}
\definecolor{mellowgreen}{RGB}{225, 242, 225}
\definecolor{mellowblue}{RGB}{225, 255, 248}
\definecolor{mellowyellow}{RGB}{255, 248, 210}
\definecolor{mellowred}{RGB}{250, 225, 225}
\newcommand{\bluehl}[1]{{\sethlcolor{mellowblue}\hl{#1}}}
\newcommand{\redhl}[1]{{\sethlcolor{mellowred}\hl{#1}}}
\algtext*{EndFor}
\def\BibTeX{{\rm B\kern-.05em{\sc i\kern-.025em b}\kern-.08em
    T\kern-.1667em\lower.7ex\hbox{E}\kern-.125emX}}

\usepackage{amsmath,amsfonts,bm, xcolor,mathtools}
\usepackage{amsthm}
\usepackage{soul}

\def\min{\mathop{\mathsf{min}}}
\def\argmin{\mathop{\mathsf{arg\,min}}} 

\DeclarePairedDelimiterX{\infdivx}[2]{(}{)}{
  #1\;\delimsize\|\;#2}

\definecolor{lightgreen}{rgb}{.9,1,.9}
\definecolor{trolleygrey}{rgb}{0.5, 0.5, 0.5}
\definecolor{BrickRed}{rgb}{0.6,0,0}
\definecolor{RoyalBlue}{rgb}{0,0,0.8}
\definecolor{Tdgreen}{rgb}{0,0.4,0.7}
\definecolor{pinegreen}{rgb}{0.0, 0.47, 0.44}
\definecolor{cornellred}{rgb}{0.7, 0.11, 0.11}
\definecolor{cadmiumgreen}{rgb}{0.0, 0.42, 0.24}
\definecolor{spirodiscoball}{rgb}{0.06, 0.75, 0.99}
\definecolor{mylightblue}{rgb}{0.85, 0.90, 0.94}
\definecolor{maroon}{cmyk}{0,0.87,0.68,0.32}
\definecolor{mydarkblue}{RGB}{0,0,128}  

\DeclareMathAlphabet{\mathsfit}{\encodingdefault}{\sfdefault}{m}{sl}
\SetMathAlphabet{\mathsfit}{bold}{\encodingdefault}{\sfdefault}{bx}{n}

\def\etabm{{\bm{\eta}}}

\def\thetabm{{\bm{\theta }}}

\def\R{\mathbb{R}}
\def\E{\mathbb{E}}

\def\zerobm{{\bm{0}}}

\def\vbm{{\bm{v}}}

\def\\px{{\bm{x}}}
\def\ubm{{\bm{u}}}
\def\zbm{{\bm{z}}}
\def\ybm{{\bm{y}}}

\def\zbm{{\bm{z}}}

\def\abm{{\bm{a}}}

\def\mbm{{\bm{m}}}
\def\bbm{{\bm{b}}}

\def\xbm{{\bm{x}}}

\def\Mbm{{\bm{M}}}

\def\Abm{{\bm{A}}}

\def\Hbm{{\mathbf{H}}}

\def\Pbm{{\bm{P}}}
\def\Fbm{{\bm{F}}}
\def\Ibm{{\bm{I}}}

\def\\px{{\bm{X}}}

\def\Sigmabm{{\bm{\Sigma}}}

\def\px{{p(\xbm)}}

\def\dd{{\mathrm{d}}}

\def\Ncal{{\mathcal{N}}}

\def\Lcal{{\mathcal{L}}}

\def\Ucal{{\mathcal{U}}}

\def\Rsf{{\mathsf{R}}}
\def\Tsf{{\mathsf{T}}}
\def\Tsf{{\mathsf{T}}}

\theoremstyle{plain}

\newtheorem{proposition}{Proposition}
\newtheorem{lemma}{Lemma}

\theoremstyle{definition}

\theoremstyle{remark}

\title{FIRM: Flow-based Imaging via Regularized Minimization}

\vspace{5em}
\author{%
\normalsize Shirin Shoushtari$^1$ \quad  
Edward P.~Chandler$^2$ \quad Xiao Shi$^2$ \quad Ulugbek S.~Kamilov$^2$\\[0.7em]
\small \textnormal{$^1$WashU, $^2$UW-Madison}\\[0.5em]
\footnotesize \texttt{s.shirin@wustl.edu, epchandler@wisc.edu, xiao.shi@wisc.edu, kamilov@wisc.edu}}

\begin{document}
\maketitle

\begin{abstract}
Flow matching methods for imaging inverse problems typically incorporate measurements through network conditioning or guidance during sampling. Neither approach explicitly applies the forward operator within the learned conditional velocity field.
We develop a principled measurement-conditional velocity parameterization that does.
For a linear interpolation path, we express the optimal velocity through the posterior mean $\E[\xbm_1|\xbm_t, \ybm]$ and show that this mean is the unique minimizer of a variational objective with an explicit data-consistency term. 
The velocity defined by this minimizer provably transports the source distribution to the measurement-conditioned posterior.
This result leads to a forward operator-aware velocity field that is trained end-to-end and requires no separate guidance during sampling.
Across five imaging tasks, our method achieves leading reconstruction quality with up to $50\times$ fewer network evaluations than competitive flow-based methods. Varying the number of sampling steps also controls the distortion-perception trade-off without retraining.
\end{abstract}

\section{Introduction}
Imaging inverse problems aim to recover an unknown image from degraded measurements. They arise in magnetic resonance imaging (MRI), computed tomography (CT), microscopy, astronomical imaging, and remote sensing~\cite{bertero2021introduction, mccann2017convolutional, ongie2020deep}. 
Because the forward operator is often underdetermined, the measurements do not uniquely determine the image. Reconstruction therefore requires prior information. A common classical approach combines a data-fidelity term with a hand-designed regularizer, such as sparsity or smoothness~\cite{tikhonov1977solutions,vogel2002computational, tarantola2005inverse, mohammad2021regularization}. Learning-based methods often replace or complement hand-designed regularization with priors learned from data.

Score-based diffusion models~\cite{ho2020denoising, song2020score} and flow-matching models~\cite{lipman2023flow, albergo2023building, liu2023flow} provide expressive learned priors for imaging inverse problems~\cite{daras2024survey}. A key distinction among generative inverse-problem solvers is how and when measurement information enters the reconstruction.

A common approach starts with a pretrained image prior $p(\xbm)$ and introduces measurements during inference. DiffPIR~\cite{zhu2023denoising} and DDRM~\cite{kawar2022denoising} use proximal or spectral data-consistency updates, while variational PnP diffusion~\cite{mardani2024variational} alternates denoising with data-consistency projections. Flow-based methods similarly combine an unconditional velocity field with gradient- or projection-based consistency updates~\cite{martin2025pnpflow,pourya2026flower,zhang2024flow}. Reusing a prior across forward models is valuable, but the resulting solver must repeatedly correct its generative trajectory using measurements at inference, often requiring many velocity evaluations to produce a reconstruction. See Appendix~\ref{app:rel_work} for further discussion.

Another approach trains a score or velocity field conditioned on the measurement $\ybm$~\cite{saharia2022image,delbracio2023inversion,liu2023i2sb}. Conditioning can eliminate separate guidance at inference, but measurement information typically enters only as a network input, such as $\ybm$, $\Abm^\Tsf\ybm$, or a learned embedding. Data consistency must therefore be learned from paired examples rather than imposed through an explicit forward operator-dependent update.

We give explicit data consistency a different role: it becomes part of each evaluation of a learned measurement-conditional velocity field. Under linear interpolation, $\xbm_t$ and $\ybm$ are conditionally independent noisy observations of the clean image $\xbm_1$, and their joint posterior mean determines the optimal conditional velocity, $\vbm(\xbm_t,t,\ybm)=\big(\E[\xbm_1|\xbm_t,\ybm]-\xbm_t\big)/(1-t)$. We characterize this posterior mean as the unique minimizer of a variational objective with interpolation consistency, explicity measurement consistency, and an implicit regularizer. We then prove that the velocity defined by this minimizer transports the source distribution to the measurement-conditioned posterior. Guided by this characterization, we alternate a forward operator-dependent data-consistency steps with a learned update to parametrize the velocity field. The full parameterization is trained end-to-end by flow matching and requires no separate guidance during sampling.

Our contributions are threefold. 
\textbf{(a)} We introduce FIRM, a novel measurement-conditional flow-matching method that explicitly incorporates data-consistency into each velocity field evaluation.
\textbf{(b)} We characterize the optimal conditional velocity through a variational minimizer and show that the resulting exact field transports the source distribution to the measurement-conditioned posterior. 
\textbf{(c)} Across five imaging tasks, our method achieves leading reconstruction quality with up to $50\times$ fewer network evaluations than competitive flow-based methods. Varying the number of sampling steps also controls the distortion-perception trade-off without retraining.

\section{Background}

\paragraph{Inverse problems.}
Imaging inverse problems seek to recover an unknown image $\xbm \in \R^{n}$ from degraded measurements $\ybm \in \R^{m}$, commonly modeled as
\begin{equation}\label{eq:forward_model}
\ybm= \Abm \xbm+ \etabm ,
\end{equation}
where $\Abm:\R^{n}\rightarrow \R^{m}$ is the measurement operator and $\etabm \sim \Ncal(\bm 0, \sigma^2 \Ibm)$ is additive Gaussian noise with standard deviation $\sigma$. Because $\Abm$ may be noninvertible or underdetermined, the measurements may not uniquely determine $\xbm$; even when the solution is unique,
ill-conditioning can make recovery unstable. Thus, reconstruction requires prior information. A common variational estimator is
\begin{equation}\label{eq:inverse_variational}
\hat{\xbm} =\argmin_{\xbm} f(\xbm) \quad\text{with}\quad f(\xbm) = g(\xbm)+h(\xbm),
\end{equation}
where $g(\xbm)$ enforces measurement consistency and $h(\xbm)$ encodes prior information. Under additive Gaussian noise, a standard choice is $g(\xbm) = \frac{1}{2}\|\Abm\xbm-\ybm\|_2^2$, while $h$ may promote properties such as total variation, wavelet sparsity, or sparse dictionary representations~\cite{rudin1992nonlinear, beck2009fast, figueiredo2003algorithm}.

\paragraph{PnP and HQS.}
Plug-and-Play (PnP) methods retain the data-fidelity term and replace the proximal update associated with a hand-designed regularizer by a learned denoiser~\cite{venkatakrishnan2013plug}.  The proximal operator is
\begin{equation}\label{eq:prox}
\text{prox}_{\gamma h}(\xbm) = \argmin_{\ubm}\left\{h(\ubm) + \frac{1}{2\gamma}\|\ubm - \xbm\|_2^2\right\}.
\end{equation}
When $h$ is a negative log-prior, this operator corresponds to MAP denoising with Gaussian noise level $\sqrt{\gamma}$. PnP methods instead use a learned denoiser to encode prior information implicitly~\cite{romano-elad-milanfar-red,ahmad2020plug,ryu2019plug, kamilov2022pnp}. 
Half-quadratic splitting (HQS) introduces an auxiliary variable $\zbm$ and separates the data-fidelity and regularization terms as
\begin{equation}
\min_{\xbm, \zbm}\left\{g(\xbm) + h(\zbm) + \frac{\mu}{2}\|\xbm - \zbm\|_2^2\right\},
\end{equation}
where $\mu>0$ controls the coupling between $\xbm$ and $\zbm$.
Alternately minimizing gives
\begin{align}
\xbm^{k+1} &= \argmin_{\xbm}\left\{g(\xbm) + \frac{\mu}{2}\|\xbm - \zbm^k\|_2^2\right\} = \left(\Abm^\top\Abm + \mu\Ibm\right)^{-1}\left(\Abm^\top\ybm + \mu\zbm^k\right), \label{eq:hqs_x}\\
\zbm^{k+1} &= \argmin_{\zbm}\left\{h(\zbm) + \frac{\mu}{2}\|\xbm^{k+1} - \zbm\|_2^2\right\} = \text{prox}_{h/\mu}\left(\xbm^{k+1}\right). \label{eq:hqs_z}
\end{align}
The $\xbm$-update enforces data consistency, while the $\zbm$-update applies prior information. PnP replaces the latter with a learned denoiser at noise level $1/\sqrt{\mu}$. Algorithm unrolling treats the alternating updates as differentiable network layers that can be trained end-to-end~\cite{gregor2010learning,monga2021algorithm}.

\paragraph{Flow matching.}
Flow matching~\cite{lipman2023flow, liu2023flow, albergo2023building} constructs generative models by learning continuous transport dynamics from a source distribution $p_0$ to a target distribution $p_1$. The intermediate distributions $\{p_t\}_{t\in[0,1]}$ satisfy the continuity equation
\begin{equation}
\partial_t p_t+\nabla \cdot\left(p_t \vbm_t\right)=0,
\label{eq:continuity}
\end{equation}
where $\vbm_t$ is the velocity field and samples evolve according to $d\xbm_t/dt = \vbm_t(\xbm_t)$. Following the linear interpolation path used in flow matching and rectified flows~\cite{lipman2023flow, liu2023flow}, we have 
\begin{equation}
\xbm_t=(1-t)\xbm_0 + t\xbm_1, \qquad t\in[0,1],
\label{eq:linear_path}
\end{equation}
where $\xbm_0\sim p_0$ and $\xbm_1\sim p_1$. The optimal marginal velocity is $\vbm(\xbm_t,t)=\E[\xbm_1-\xbm_0|\xbm_t]$. Because this conditional expectation is generally unavailable, conditional flow matching (CFM)~\cite{lipman2023flow,tong2024improving} regresses directly onto the sample-wise path velocity,
\begin{equation}
\Lcal_{\mathrm{CFM}}(\theta)
=\E_{t,\xbm_0,\xbm_1}\left[\|\vbm_\thetabm(\xbm_t,t)-(\xbm_1-\xbm_0)\|^2\right],
\label{eq:cfm_loss}
\end{equation}
whose gradient matches that of the marginal flow-matching objective. Here, ``conditional'' refers to the probability path used to construct the training target, not to measurement conditioning. For the linear path, the optimal velocity can also be written as
\begin{equation}
\vbm(\xbm_t,t)=\E\left[\xbm_1- \xbm_0|\xbm_t\right]=\frac{\E\left[\xbm_1|\xbm_t\right]-\xbm_t}{1-t}.
\label{eq:velocity_posterior}
\end{equation}
Thus, velocity estimation reduces to conditional-mean estimation. The sample-wise paths are straight, motivating numerical integration with relatively few steps in rectified-flow models~\cite{liu2023flow,tong2024improving}. After training, samples are generated by solving
\begin{equation}
\hat{\xbm}_1=\xbm_0+\int_0^1\vbm_\thetabm(\xbm_t,t)\,\dd t.
\label{eq:sampling}
\end{equation}
toward $t=1$. Numerical solvers evaluate the representation in~\eqref{eq:velocity_posterior} only for $t<1$, where it remains well defined.

\section{Method}
\label{sec:method}

\subsection{Theoretical Characterization}
\label{sec:theory}

Under linear interpolation, the interpolation state $\xbm_t$ and measurement $\ybm=\Abm\xbm_1+\etabm$ are conditionally independent noisy observations of $\xbm_1$. They form the augmented observation model
\begin{equation}
    \underbrace{
        \begin{bmatrix}
            \xbm_t\\
            \ybm
        \end{bmatrix}
    }_{\widetilde{\ybm}}
    =
    \underbrace{
        \begin{bmatrix}
            t\Ibm\\
            \Abm
        \end{bmatrix}
    }_{\widetilde{\Abm}}
    \xbm_1
    +
    \underbrace{
        \begin{bmatrix}
            (1-t)\xbm_0\\
            \etabm
        \end{bmatrix}
    }_{\widetilde{\bbm}},
    \qquad
    \widetilde{\bbm}
    \sim
    \Ncal\left(
        \zerobm,
        \operatorname{diag}\bigl((1-t)^2\Ibm,\sigma^2\Ibm\bigr)
    \right).
    \label{eq:joint_obs}
\end{equation}

The MMSE estimator of $\xbm_1$ under the augmented observation model ($\widetilde{\ybm} = \widetilde{\Abm}\xbm_1 + \widetilde{\bbm}$) is the posterior mean $\E[\xbm_1|\xbm_t,\ybm]$. The following proposition characterizes the posterior mean as the minimizer of a variational objective and shows that the velocity it defines transports the source distribution to the measurement-conditioned posterior.

\begin{proposition}\label{prop:conditional_flow}
For a non-degenerate $p_1$ with bounded support, 
$\sigma>0$, and $t\in(0,1)$:
\begin{enumerate}[label=(\alph*)]
    \item There exists a regularizer $\phi_t:\mathbb{R}^n\to\mathbb{R}\cup\{+\infty\}$, unique up to an additive constant, such that $\E[\xbm_1|\xbm_t,\ybm]$ is the unique global minimizer, and the unique stationary point of
    \begin{equation}
    \E[\xbm_1|\xbm_t,\ybm]
    =\argmin_{\xbm\in\mathbb{R}^n}
    \frac{1}{2(1-t)^2}\|\xbm_t-t\xbm\|_2^2
    +\frac{1}{2\sigma^2}\|\ybm-\Abm\xbm\|_2^2
    +\phi_t(\xbm),
    \label{eq:mmse_variational}
    \end{equation}
where the minimizer determines the optimal measurement-conditional velocity,
\begin{equation}\label{eq:conditional_velocity}
\vbm(\xbm_t,t,\ybm)= \frac{\E[\xbm_1|\xbm_t,\ybm]- \xbm_t}{1-t}.
\end{equation}

\item  Moreover, $\vbm$ satisfies the conditional continuity equation
\begin{equation}
\partial_t p_t(\xbm\mid\ybm) +\nabla\cdot\left(p_t(\xbm\mid\ybm)\vbm(\xbm,t,\ybm)\right) =0.
\end{equation}
For every $\tau<1$, the corresponding flow transports the Gaussian source to
$p_\tau(\cdot\mid\ybm)$, and
$p_\tau(\cdot\mid\ybm)\Rightarrow p(\xbm_1\mid\ybm)$ as
$\tau\rightarrow1^{-}$.
\end{enumerate}
\end{proposition}

The proof is  provided in Appendix~\ref{ap_sec:proof_pros1}. The variational statement applies the MMSE characterization of~\cite{gribonval2013reconciling} to the augmented model in~\eqref{eq:joint_obs}. In~\eqref{eq:mmse_variational}, two quadratic terms enforce interpolation and measurement consistency, while $\phi_t$ captures the remaining structure. The minimizer yields the posterior mean and hence the optimal conditional velocity, whose probability flow transports the source distribution to the measurement-conditioned posterior. Because $\phi_t$ is unavailable, the next subsection turns this principled characterization into a practical algorithm.

\subsection{Operator-Aware Conditional Velocity}
\label{sec:param}

Proposition~\ref{prop:conditional_flow} reduces the construction of the optimal conditional velocity to minimizing~\eqref{eq:mmse_variational}. However, the implicit regularizer $\phi_t$ is not available in closed form. We therefore apply HQS to separate the objective into an explicit data-consistency update and a learned update associated with $\phi_t$. Introducing an auxiliary variable $\zbm$ and a quadratic penalty gives
\begin{equation}
\min_{\xbm,\zbm}\frac{1}{2(1-t)^2}\|\xbm_t-t\zbm\|_2^2+\frac{1}{2\sigma^2}\|\ybm-\Abm\xbm\|_2^2+\phi_t(\zbm)+\frac{\mu}{2}\|\xbm-\zbm\|_2^2,
\label{eq:hqs_problem}
\end{equation}
where $\mu>0$ controls the coupling. Alternating minimization over $\xbm$ and $\zbm$ yields
\begin{align}
\xbm^{k+1}&= \argmin_{\xbm}\frac{1}{2\sigma^2}\|\ybm-\Abm\xbm\|_2^2
+\frac{\mu}{2}\|\xbm-\zbm^{k}\|_2^2,
\label{eq:x_update}\\
\zbm^{k+1}&=\argmin_{\zbm}\phi_t(\zbm)
+\frac{1}{2(1-t)^2}\|\xbm_t-t\zbm\|_2^2
+\frac{\mu}{2}\|\xbm^{k+1}-\zbm\|_2^2.
\label{eq:z_update}
\end{align}
The $\xbm$-update explicitly enforces consistency with the measurement model, while the $\zbm$-update combines interpolation consistency with implicit regularization. Because $\phi_t$ is unknown, we implement the $\zbm$-update using a learned estimator. The complete implementation of the algorithm is provided in Appendix~\ref{app:alg_details} and Algorithm~\ref{alg:posterior_mean}.

Let $\zbm_{\thetabm}^{K}(\xbm_t,t,\ybm)$ denote the output after $K$ iterations. The output estimates the measurement-conditioned posterior mean $ \E[\xbm_1|\xbm_t,\ybm]$ and defines the learned velocity
\begin{equation}
    \vbm_{\thetabm}(\xbm_t,t,\ybm)=
    \frac{ \zbm_{\thetabm}^{K}(\xbm_t,t,\ybm)-\xbm_t}{ 1-t}.
    \label{eq:induced_velocity}
\end{equation}
The forward operator therefore enters the learned velocity through the data-consistency update at every iteration.

\subsection{Training the Conditional Velocity Field}\label{sec:training}

Given a clean image $\xbm_1$, we generate a measurement
$\ybm=\Abm\xbm_1+\etabm$, sample
$\xbm_0\sim\Ncal(\bm 0,\Ibm)$ independently, and draw
$t\sim\Ucal[t_{\min},t_{\max}]$. The interpolation state is
\begin{equation}
    \xbm_t=t\xbm_1+(1-t)\xbm_0.
\end{equation}

Algorithm~\ref{alg:posterior_mean} produces the operator-aware estimate
$\zbm_{\thetabm}^K(\xbm_t,t,\ybm)$ of the measurement-conditioned posterior
mean. We use the stabilized time scale
\begin{equation}
    \tau_t=\max(1-t,\tau_{\min}),
\end{equation}
where $t_{\max}<1$ excludes the singular endpoint and $\tau_{\min}>0$ bounds
the velocity scaling. The predicted and target velocities are
\begin{equation}
    \vbm_{\thetabm}(\xbm_t,t,\ybm)
    =
    \frac{\zbm_{\thetabm}^K(\xbm_t,t,\ybm)-\xbm_t}{\tau_t},
    \qquad
    \vbm_{\mathrm{target}}
    =
    \frac{\xbm_1-\xbm_t}{\tau_t}, 
    \label{eq:training_velocities}
\end{equation}
where $\zbm_{\thetabm}^K(\xbm_t,t,\ybm)$ is the posterior mean provided by Algorithm~\ref{alg:posterior_mean}. We train the complete estimator end-to-end with the weighted flow-matching
loss
\begin{equation}
    \Lcal(\thetabm)
    =
    \E_{\xbm_1,\xbm_0,t,\ybm}
    \left[
        \omega(t)
        \left\|
            \vbm_{\thetabm}(\xbm_t,t,\ybm)
            -
            \vbm_{\mathrm{target}}
        \right\|_2^2
    \right],
    \qquad
    \omega(t)
    =
    \frac{1+\tau_t^3}{\tau_t}.
    \label{eq:velocity_loss}
\end{equation}
Algorithm~\ref{alg:training} summarizes the training procedure.

\subsection{Sampling with the Conditional Velocity Field}
\label{sec:sampling}

At inference, we use the theoretical scale $1-t$ rather than the training
floor $\tau_t$. For a state $\xbm_t$, Algorithm~\ref{alg:posterior_mean}
returns $\zbm_{\thetabm}^K(\xbm_t,t,\ybm)$, and the learned conditional
velocity is computed using \eqref{eq:induced_velocity}.
Starting from $\xbm_0\sim\Ncal(\bm 0,\Ibm)$, we integrate
\begin{equation}
    \frac{\dd\xbm_t}{\dd t}
    =
    \vbm_{\thetabm}(\xbm_t,t,\ybm)
\end{equation}
using $N$ explicit Euler steps. With $\Delta t=1/N$ and
$t_i=i\Delta t$, the update is
\begin{align}
    \zbm_{\thetabm,i}^K
    &=
    \zbm_{\thetabm}^K(\xbm_i,t_i,\ybm),
    \\
    \vbm_i
    &=
    \frac{\zbm_{\thetabm,i}^K-\xbm_i}{1-t_i},
    \\
    \xbm_{i+1}
    &=
    \xbm_i+\Delta t\,\vbm_i,
    \qquad i=0,\ldots,N-1.
    \label{eq:euler_sampling}
\end{align}
Since $t_i<1$ for every velocity evaluation, the denominator remains
well-defined. Because the posterior-mean estimator already incorporates the
forward operator, sampling requires no external guidance. Changing $N$ provides
test-time control over the distortion--perception trade-off without retraining.
Algorithm~\ref{alg:sampling} summarizes the sampling procedure.

\section{Experimental Results}
\label{sec:experiments}

\begin{figure*}[t]
\centering
\begin{minipage}[t]{0.48\textwidth}
\begin{algorithm}[H]
\small 
\caption{Training}
\label{alg:training}
\begin{algorithmic}[1]
        \Require $\Abm$, $\sigma$,
        $[t_{\min},t_{\max}]$, and $\tau_{\min}$
        \Repeat
            \State $\xbm_1\sim p_1$, $\xbm_0\sim\Ncal(\bm 0,\Ibm)$,
            $t\sim\Ucal[t_{\min},t_{\max}]$
            \State $\etabm\sim\Ncal(\bm 0,\sigma^2\Ibm)$ , 
            $\ybm\gets\Abm\xbm_1+\etabm$
            \State $\xbm_t\gets t\xbm_1+(1-t)\xbm_0$, $\tau_t= \max(1-t,\tau_{\min})$
            \State $\zbm_{\thetabm}^K\gets
            \Call{PostMean}{\xbm_t,t,\ybm,\Abm}$ \Comment{\colorbox{blue!8}{\scriptsize Alg.~\ref{alg:posterior_mean}}}
            \State $\vbm_{\thetabm}\gets
            (\zbm_{\thetabm}^K-\xbm_t)/\tau_t$
            \State $\vbm_{\mathrm{target}}\gets
            (\xbm_1-\xbm_t)/\tau_t$
            \State Update $\thetabm$ using~\eqref{eq:velocity_loss}
        \Until{convergence}
    \end{algorithmic}
\end{algorithm}
\end{minipage}
\hfill
\begin{minipage}[t]{0.48\textwidth}
\begin{algorithm}[H]
\caption{Sampling}
\small
\label{alg:sampling}
\begin{algorithmic}[1]
        \Require Measurement $\ybm$, forward operator $\Abm$, trained model,
        and step count $N$
        \State Sample $\xbm_0\sim\Ncal(\bm 0,\Ibm)$
        \State $\Delta t\gets1/N$
        \For{$i=0,\ldots,N-1$}
            \State $t_i\gets i\Delta t$
            \State $\zbm_{\thetabm,i}^K\gets
           \Call{PostMean}{\xbm_i,t_i,\ybm,\Abm}$ \Comment{\colorbox{blue!8}{\scriptsize Alg.~\ref{alg:posterior_mean}}}
            \State $\vbm_i\gets
            (\zbm_{\thetabm,i}^K-\xbm_i)/(1-t_i)$
            \State $\xbm_{i+1}\gets\xbm_i+\Delta t\,\vbm_i$
        \EndFor
        \State \Return $\xbm_N$
    \end{algorithmic}
\end{algorithm}
\end{minipage}
\end{figure*}

\begin{figure}[t]
\centering
\includegraphics[width=\textwidth]{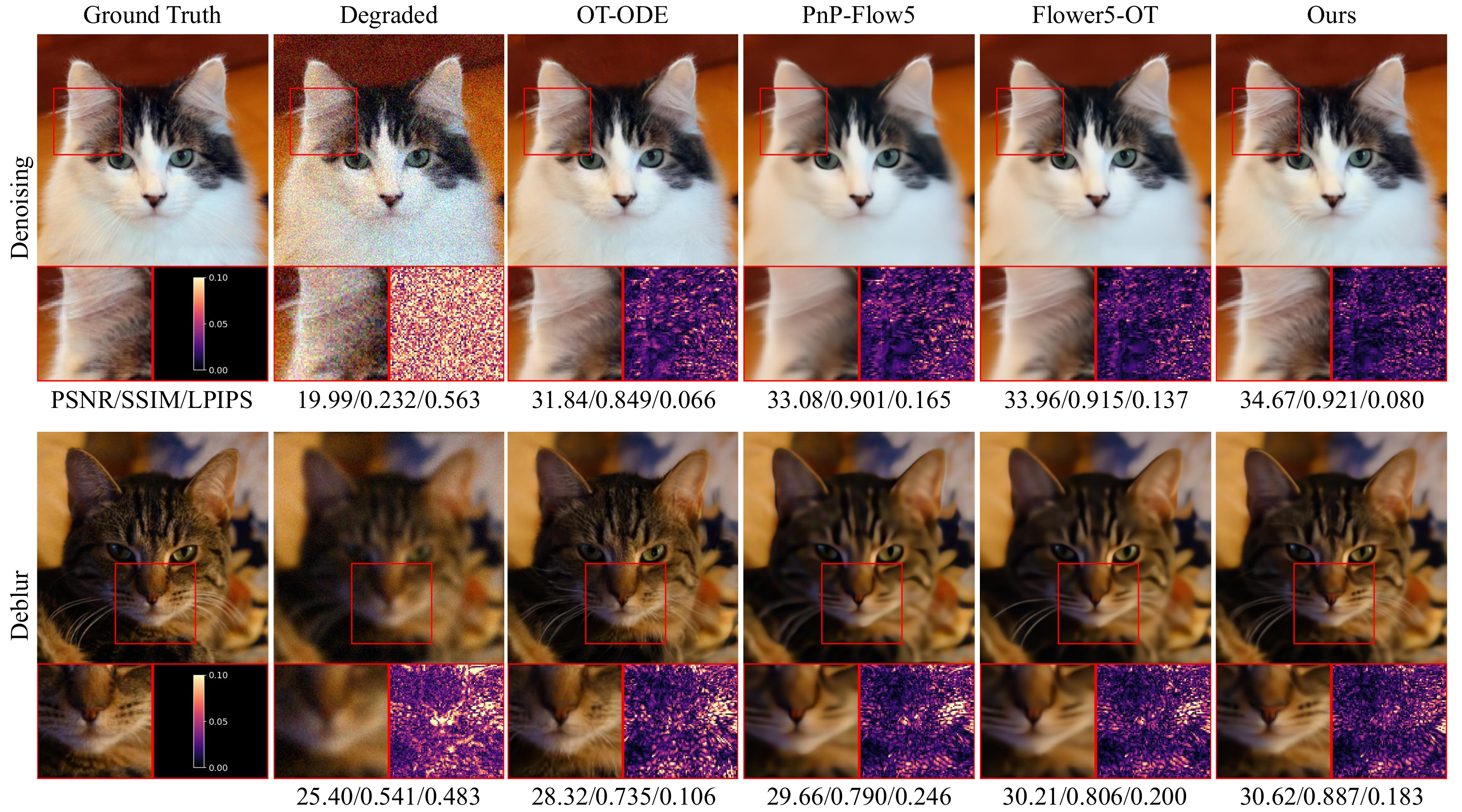}  \vspace{-1.8em}
\caption{Qualitative reconstructions on AFHQ-Cat for denoising (top) and
deblurring (bottom). Values below each panel are PSNR/SSIM/LPIPS. Our method preserves fine image details and produces
faithful reconstructions using only $N{=}2$ sampling steps.}
\label{fig:qual_afhq}
\end{figure}

\begin{table}[t]
\centering
\caption{Quantitative results on 100 CelebA test images. Best and second-best results are highlighted in \bluehl{blue} and \redhl{red}, respectively. Our method achieves the best reconstruction quality while attaining competitive perceptual quality.}
\label{tab:repro_celeba}
\setlength{\tabcolsep}{3.5pt}
\renewcommand{\arraystretch}{1.12}
\resizebox{\textwidth}{!}{
\begin{tabular}{l *{15}{c}}
\toprule
\multirow{2}{*}{\textbf{Method}}
& \multicolumn{3}{c}{\textbf{Denoising}}
& \multicolumn{3}{c}{\textbf{Deblurring}}
& \multicolumn{3}{c}{\textbf{Super-resolution}}
& \multicolumn{3}{c}{\textbf{Random inpainting}}
& \multicolumn{3}{c}{\textbf{Box inpainting}} \\
\cmidrule(lr){2-4}
\cmidrule(lr){5-7}
\cmidrule(lr){8-10}
\cmidrule(lr){11-13}
\cmidrule(lr){14-16}
& PSNR$\uparrow$ & SSIM$\uparrow$ & LPIPS$\downarrow$
& PSNR$\uparrow$ & SSIM$\uparrow$ & LPIPS$\downarrow$
& PSNR$\uparrow$ & SSIM$\uparrow$ & LPIPS$\downarrow$
& PSNR$\uparrow$ & SSIM$\uparrow$ & LPIPS$\downarrow$
& PSNR$\uparrow$ & SSIM$\uparrow$ & LPIPS$\downarrow$ \\
\midrule
Degraded    & 20.00 & 0.348 & 0.371 & 27.81 & 0.740 & 0.125 & 10.25 & 0.183 & 0.827 & 11.95 & 0.196 & 1.041 & 22.26 & 0.743 & 0.213 \\
PnP-GS      & 32.54 & 0.908 & 0.035 & 33.97 & 0.924 & 0.041 & 31.23 & 0.890 & 0.065 & 29.20 & 0.875 & 0.069 & -- & -- & -- \\
PnP-HQS      & 
32.17 & 0.897 & 0.044 & 
34.87 &0.946  & 0.033 & 
32.33 &	0.911 & 0.075 & 
32.81 &	0.945 & 0.023 & -- & -- & -- \\
ISTA-Net     & 
- & - & - & 
35.38 &0.950 & 0.035 & 
33.54 &0.937& 0.038 & 
\redhl{34.70} & 0.960 & 0.022 & -- & -- & -- \\

DiffPIR     & 31.12 & 0.883 & 0.060 & 32.68 & 0.910 & 0.060 & 31.46 & 0.893 & 0.034 & 31.69 & 0.916 & 0.025 & -- & -- & -- \\
I$^2$SB & 
- & - & - & 
33.25 & 0.924& \bluehl{0.012} & 
31.92& 0.891 &\bluehl{0.014} &
32.80 & 0.938 & \bluehl{0.011} & 
28.72 & 0.803 & 0.042\\
OT-ODE      & 30.49 & 0.858 & 0.032 & 32.96 & 0.920 & 0.029 & 31.34 & 0.903 & 0.027 & 28.65 & 0.870 & 0.051 & 29.37 & 0.919 & 0.038 \\
D-Flow      & 26.01 & 0.606 & 0.092 & 31.21 & 0.854 & 0.037 & 30.44 & 0.843 & {0.026} & 33.61 & 0.942 & 0.015 & 30.59 & 0.898 & 0.027 \\
Flow-Priors & 29.31 & 0.767 & 0.135 & 31.51 & 0.857 & 0.056 & 28.36 & 0.715 & 0.100 & 32.87 & 0.943 & 0.019 & 30.06 & 0.859 & 0.048 \\
PnP-Flow1   & 31.75 & 0.904 & 0.044 & 34.48 & 0.936 & 0.039 & 31.04 & 0.901 & 0.045 & 33.02 & 0.944 & 0.018 & 30.45 & 0.933 & 0.037 \\
PnP-Flow5   & 32.24 & 0.910 & 0.056 & 34.80 & 0.941 & 0.046 & 31.44 & 0.905 & 0.056 & 33.95 & 0.953 & 0.022 & 31.06 & 0.939 & 0.043 \\
Flower1-OT  & 32.22 & 0.913 & 0.033 & 34.95 & 0.947 & {0.025} & 32.30 & 0.922 & 0.034 & 33.05 & 0.944 & 0.017 & 31.17 & 0.945 & 0.022 \\
Flower5-OT  & 33.07 & 0.925 & 0.038 & 35.65 & 0.954 & 0.031 & 33.03 & 0.931 & 0.039 & 33.92 & 0.953 & 0.020 & 31.85 & 0.952 & 0.023 \\
\midrule
Ours $(N=2)$    & \bluehl{33.67} & \bluehl{0.932} & 0.034 & 
\bluehl{36.04} & \bluehl{0.957} & 0.029 & 
\bluehl{34.39} & \bluehl{0.949} & 0.028& 
\bluehl{35.03} & \bluehl{0.963} & 0.018 & 
\bluehl{33.25} & \bluehl{0.960} & 0.022 \\

Ours $(N=4)$    & \redhl{33.24} & \redhl{0.928} & 0.034 & 
\redhl{35.78} & \redhl{0.956} & 0.028 & 
\redhl{33.96} & \redhl{0.946} &{0.026} & 
{34.67} &  \redhl{0.962} & 0.016 & 
\redhl{32.57} & \redhl{0.958} & \redhl{0.021} \\

Ours $(N=10)$    & 32.26 & 0.914 & \redhl{0.032} & 
35.03 & 0.950 & \redhl{0.023} & 
33.13 & 0.937 & \redhl{0.022} & 
33.98 & 0.956 & {0.014} & 
31.85 & 0.953 & \bluehl{0.020} \\

Ours $(N=25)$ & 
31.56 & 0.901 & \bluehl{0.028} & 
31.46& 0.848 & 0.026 & 
30.39 & 0.837 & 0.029& 
33.42 & 0.951 & \redhl{0.013} & 
31.33 & 0.947 & \bluehl{0.020} \\
\bottomrule
\end{tabular}}
\end{table}
We evaluate our method on five inverse problems: denoising, deblurring, super-resolution, random-mask inpainting, and box inpainting. Experiments use CelebA~\cite{liu2015celeba} and AFHQ-Cat~\cite{choi2020stargan}, following the datasets and evaluation protocols of~\cite{martin2025pnpflow} and~\cite{pourya2026flower} for a fair comparison.

CelebA contains aligned face images that we center-crop and resize to $128\times128$, using the standard training, validation, and test splits. We resize AFHQ-Cat images to $256\times256$. Because AFHQ-Cat has no official validation split, we follow~\cite{martin2025pnpflow} and reserve 32 test images for validation. All images are normalized to $[-1,1]$.

We compare against state-of-the-art PnP, unfolding, diffusion-based, and flow-based inverse solvers: PnP-GS~\cite{hurault2022gradient}, PnP-HQS (DPIR)~\cite{zhang2021plug}, ISTA-Net~\cite{zhang2018ista}, DiffPIR~\cite{zhu2023denoising}, I$^2$SB~\cite{liu2023i2sb}, OT-ODE~\cite{pokle2024training}, D-Flow~\cite{benhamu2024dflow}, Flow Priors~\cite{zhang2024flow}, PnP-Flow~\cite{martin2025pnpflow}, and Flower~\cite{pourya2026flower}. All baselines except ISTA-Net and I$^2$SB enforce measurement consistency at inference time using a pretrained prior.

\begin{figure}[t]
\centering
\includegraphics[width=\textwidth]{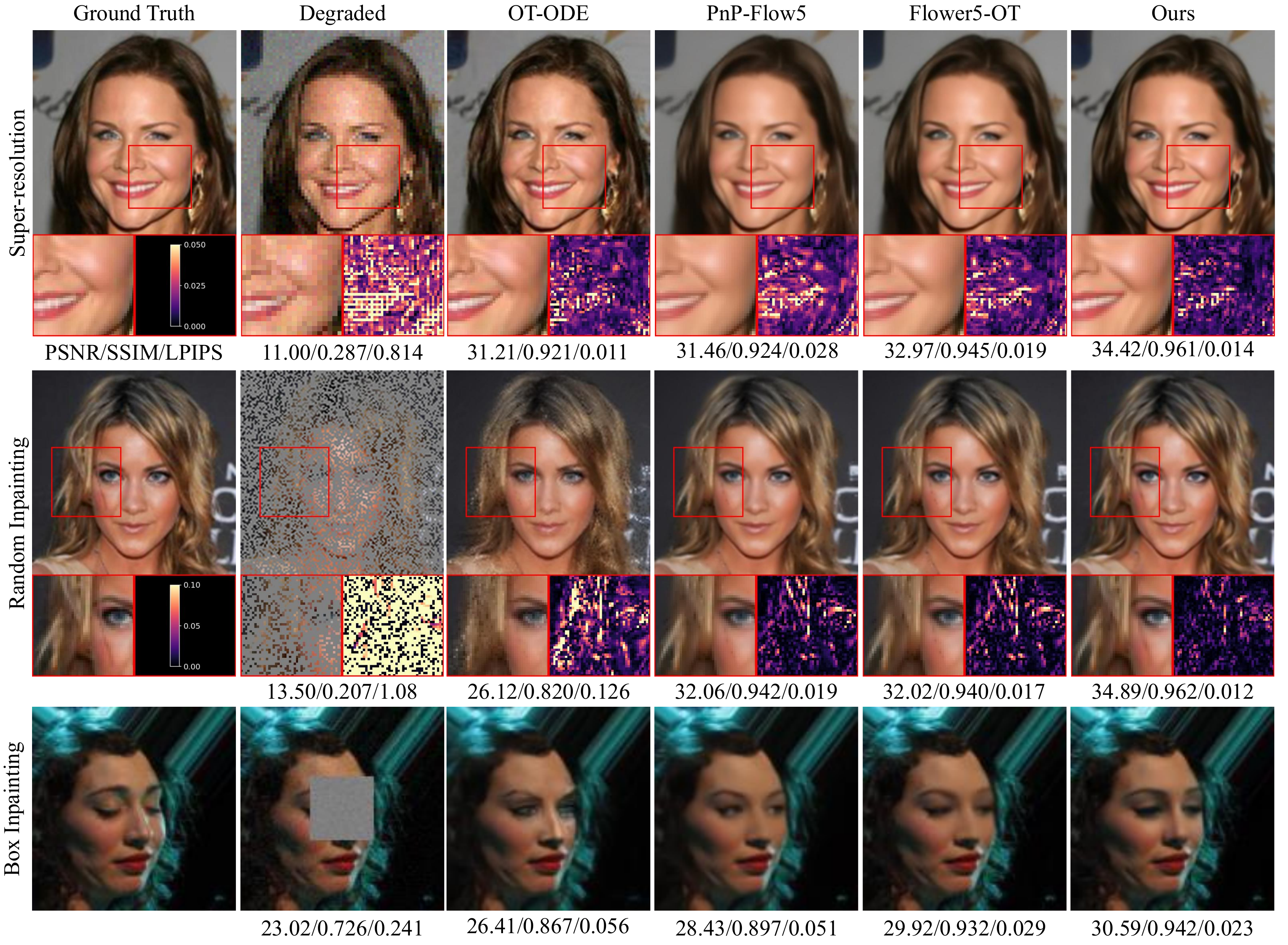}
  \vspace{-1.5em}
\caption{Qualitative reconstructions on CelebA. Values below each panel report PSNR/SSIM/LPIPS. Note that our method recovers details other baselines fail to recover while requiring only $N{=}2$ ODE steps.}
\label{fig:qual_celeba}
\end{figure}

\begin{figure}[t]
    \centering
    \includegraphics[width=\linewidth]{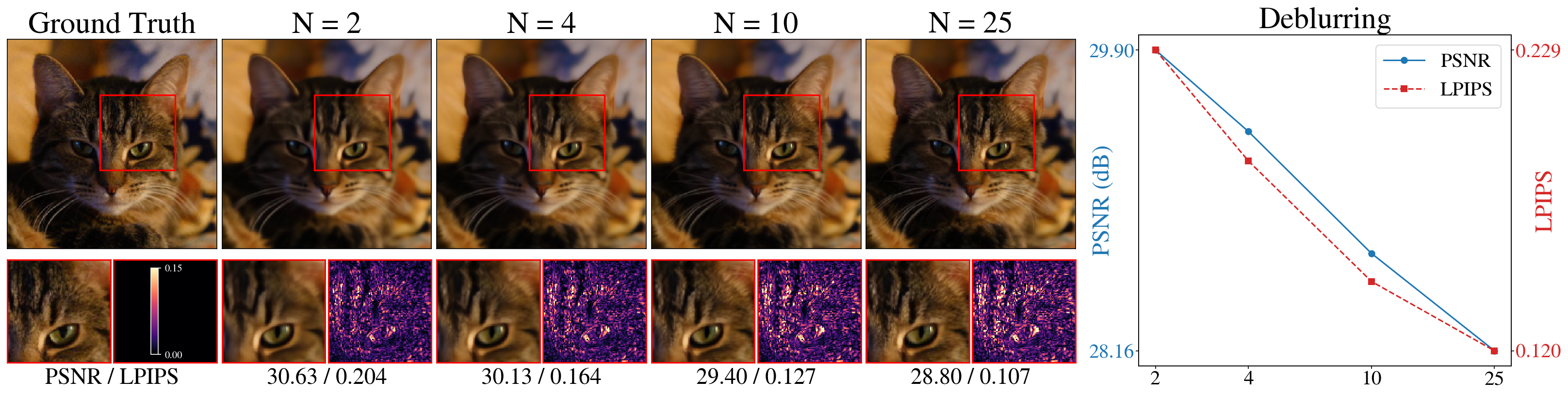}     
    \vspace{-1.5em}
\caption{Effect of the number of ODE steps $N$ on AFHQ-Cat deblurring. Values below each reconstruction report PSNR/LPIPS for the displayed image, while the right panel reports averages over the test set. Increasing $N$ improves perceptual quality at the cost of lower PSNR, allowing a single trained model to control the distortion-perception tradeoff at test time.}
    \label{fig:step_ablation_vis1}
\end{figure}

\subsection{Main Results}\label{sec:main_results}
We evaluate on five restoration tasks and report PSNR ($\uparrow$), SSIM ($\uparrow$), and LPIPS ($\downarrow$) \cite{zhang2018unreasonable}, averaged over 100 test images. We adopt the degradation settings of \cite{martin2025pnpflow} and \cite{pourya2026flower} without modification. Denoising uses Gaussian noise with $\sigma = 0.2$. Deblurring uses a $61 \times 61$ Gaussian kernel, with $\sigma_b = 1.0$ on CelebA and $\sigma_b = 3.0$ on AFHQ-Cat, and $\sigma = 0.05$. Super-resolution uses $2\times$ downsampling on CelebA and $4\times$ on AFHQ-Cat, with $\sigma = 0.05$. Random inpainting masks $70\%$ of pixels with $\sigma = 0.01$, and box inpainting applies a centered mask of $40 \times 40$ on CelebA and $80 \times 80$ on AFHQ-Cat with $\sigma = 0.05$. In tables, ``--'' marks settings where a baseline does not apply.

Tables~\ref{tab:repro_celeba} and~\ref{tab:repro_afhq_cat} report results on CelebA and AFHQ-Cat. Our method ranks first or second in most settings. Varying the number of sampling steps controls the distortion-perception trade-off, where fewer steps favor PSNR and SSIM, whereas additional steps improve LPIPS, allowing one trained model to span both regimes without retraining. Figures~\ref{fig:qual_afhq} and~\ref{fig:qual_celeba} show representative reconstructions. Following~\cite{pourya2026flower}, we also evaluate ours and the faster baselines on \emph{larger testsets} in
Tables~\ref{tab:repro_celeba1000} and ~\ref{tab:repro_afhq_cat400}.

\begin{table}[t]
\centering
\caption{Quantitative results on 100 AFHQ-Cat test images. Best and second-best results are highlighted in \bluehl{blue} and \redhl{red}, respectively. Our method achieves the best reconstruction quality while attaining the best or competitive perceptual quality.}
\setlength{\tabcolsep}{3.5pt}
\renewcommand{\arraystretch}{1.12}
\resizebox{\textwidth}{!}{
\begin{tabular}{l *{15}{c}}
\toprule
\multirow{2}{*}{\textbf{Method}}
& \multicolumn{3}{c}{\textbf{Denoising}}
& \multicolumn{3}{c}{\textbf{Deblurring}}
& \multicolumn{3}{c}{\textbf{Super-resolution}}
& \multicolumn{3}{c}{\textbf{Random inpainting}}
& \multicolumn{3}{c}{\textbf{Box inpainting}} \\
\cmidrule(lr){2-4}
\cmidrule(lr){5-7}
\cmidrule(lr){8-10}
\cmidrule(lr){11-13}
\cmidrule(lr){14-16}
& PSNR$\uparrow$ & SSIM$\uparrow$ & LPIPS$\downarrow$
& PSNR$\uparrow$ & SSIM$\uparrow$ & LPIPS$\downarrow$
& PSNR$\uparrow$ & SSIM$\uparrow$ & LPIPS$\downarrow$
& PSNR$\uparrow$ & SSIM$\uparrow$ & LPIPS$\downarrow$
& PSNR$\uparrow$ & SSIM$\uparrow$ & LPIPS$\downarrow$ \\
\midrule
Degraded & 20.00 & 0.293 & 0.529 & 24.38 & 0.529 & 0.436 & 11.59 & 0.212& 0.867 & 13.23 & 0.213 & 1.082 & 21.52 & 0.727& 0.215 \\
PnP-GS & \redhl{33.00} & 0.895 & \redhl{0.078} & 28.39 & 0.787 & 0.387 & 24.44 & 0.639 & 0.411 & 29.82 & 0.844 & 0.143 & - & - & - \\
PnP-HQS & 
32.69 & 0.896 & 0.108 & 
29.34 & 0.793 & 0.362 & 
26.58 & 0.724 & 0.441 & 
33.07 &	0.916 & 0.051 &
 - & - & - \\
ISTA-Net & 
- & - & - & 
29.36 &  0.790 & 0.342 & 
28.22 & 0.784 & 0.297 & 
34.10 &  0.929 &  0.049& 
- & - & - \\

DiffPIR & 31.05 & 0.839 & 0.186 & 28.24 & 0.747 & 0.319 & 24.24 & 0.650 & 0.385 & 32.31 & 0.886 & 0.057 & - & - & - \\
I$^2$SB & - & - & - & 
27.72 & 0.725 & \bluehl{0.077}  & 
27.19 &  0.739 & \bluehl{0.070}&
31.79 & 0.879 & \bluehl{0.026} &
27.32 & 0.763 & 0.072 \\
OT-ODE & 30.48 & 0.818 & 0.081 & 27.82 & 0.735 &{0.126} & 26.71 & 0.737 & \redhl{0.108} & 29.99 & 0.849 & 0.087 & 24.47 & 0.873 & 0.095 \\
D-Flow & 26.44 & 0.573 & 0.180 & 28.60 & 0.746 & 0.167 & 25.20 & 0.616 & 0.188 & 32.79 & 0.898 & 0.044 & 27.02 & 0.840 & 0.081 \\
Flow-Priors & 29.67 & 0.756 & 0.165 & 27.28 & 0.726 & 0.188 & 23.95 & 0.566 & 0.274 & 32.95 & 0.909 & 0.051 & 26.13 & 0.809 & 0.130 \\
PnP-Flow1 & 31.62 & 0.866 & 0.139 & 28.71 & 0.779 & 0.285 & 27.56 & 0.780 & 0.158 & 33.71 & 0.923 & 0.038 & 26.45 & 0.895 & 0.111 \\
PnP-Flow5 & 31.87 & 0.868 & 0.170 & 29.01 & 0.785 & 0.312 & 28.01 & 0.791 & 0.167 & \redhl{34.44} & \redhl{0.932} & 0.045 & 27.14 & 0.900 & 0.127 \\
Flower1-OT & 32.12 & 0.881 & 0.113 & 29.38 & 0.793 &0.237 & 26.76& 0.759 & 0.262 & 33.68 & 0.922 & 0.042 & 26.67 & 0.914 & 0.065\\
Flower5-OT & 32.75 & 0.892 & 0.126 & \redhl{29.73} & \redhl{0.801} &0.264 & 27.09 & 0.767& 0.272 & 34.36 & {0.931}& 0.048 & 27.32 & \redhl{0.922} & 0.067\\
\midrule
Ours $(N=2)$ & \bluehl{33.31}&\bluehl{0.904}&0.081& \bluehl{29.89} & \bluehl{0.809} & 0.229&
\bluehl{29.10} & \bluehl{0.816} & 0.167& 
\bluehl{34.71} & \bluehl{0.936} & 0.043& 
\bluehl{29.27} & \bluehl{0.924} & 0.060\\

Ours $(N=4)$ & 32.93&\redhl{0.898}&0.079& 
29.42 & 0.796 & 0.188 &
\redhl{28.50}&\redhl{0.801}& 0.141& 
34.30 & 0.932& 0.039& 
\redhl{28.65} & 0.921 & \redhl{0.052}\\

Ours $(N=10)$&31.96 &  0.879 &  \bluehl{0.072}& 
28.72 & 0.775 & 0.145 &
27.75&0.777& {0.115}& 
33.53&0.922&{0.033}&
27.94&0.916&\bluehl{0.047}\\

Ours $(N=25)$ & 29.69&  0.776&  0.081 & 
28.16 & 0.754 & \redhl{0.119} &
26.54&0.676& {0.121}& 
32.91 & 0.912 & \redhl{0.030} & 
26.76 & 0.819 & {0.077} \\
\bottomrule
\end{tabular}}
\label{tab:repro_afhq_cat}
\end{table}
\subsection{Inference Cost and Scalability}
\label{sec:inference-cost}

Table~\ref{tab:efficiency} reports inference time, peak GPU memory, and number of function evaluations (NFEs) for CelebA deblurring, averaged over 10 images on a single NVIDIA RTX A6000. Our inference cost is $O(NK)$, where $N$ is the number of ODE steps and $K{=}5$ is the number of inner HQS iterations; thus, the $N{=}2$ configuration requires 10 NFEs. Although its per-NFE cost is comparable to PnP-Flow and Flower, our method requires $50\times$ fewer evaluations. At the configuration achieving the best distortion metrics, it reconstructs an image in $0.24$\,s, compared with $13.03$\,s for Flower5-OT. This speedup stems from learning data consistency during training rather than enforcing it at inference.
Our method uses more peak memory than PnP-Flow and Flower because it retains intermediate iterates, but substantially less than other baselines.

\begin{table}[t]
\centering
\caption{Computation time, memory usage, and number of function evaluations (NFEs) for the
deblurring task on CelebA ($128\times128$). Lower is better for all metrics. \textbf{Best} and \textbf{second} best results are shown in \bluehl{blue} and \redhl{red} .}
\label{tab:efficiency}
\setlength{\tabcolsep}{4pt}
\resizebox{\textwidth}{!}{
\begin{tabular}{l ccccccc cccc}
\toprule
& \multicolumn{7}{c}{\textbf{Baselines}} & \multicolumn{4}{c}{\textbf{Ours}} \\
\cmidrule(lr){2-8} \cmidrule(lr){9-12}
& OT-ODE & D-Flow & Flow Priors & PnP-Flow$1$ & PnP-Flow$5$ & Flower$1$ & Flower$5$
& $N{=}2$ & $N{=}4$ & $N{=}10$ & $N{=}25$ \\
\midrule
Time (s) $\downarrow$          & 6.37 & 299.25 & 41.27 & 2.33 & 11.46 & 2.67 & 13.03
                               & \bluehl{0.24} & \redhl{0.50} & 1.21 & 3.13 \\
Peak mem (MB) $\downarrow$     & 743  & 6124   & 2728  & \bluehl{192}  & \bluehl{192}   & \redhl{193}  & \redhl{193}
                               & 331  & 331  & 331  & 331 \\
NFEs $\downarrow$              & 180  & --     & 100   & 100  & 500   & 100  & 500
                               & \bluehl{10}  & \redhl{20}   & 50   & 125 \\
\bottomrule
\end{tabular}}
\end{table}

\subsection{Ablation Studies}\label{sec:ablations}

\textbf{Number of sampling steps $N$.}
Figure~\ref{fig:step_ablation_vis1} shows how the number of sampling steps affects reconstruction quality. Across most tasks, increasing $N$ lowers PSNR but improves LPIPS. Coarse integration produces posterior-mean-like reconstructions, while finer integration moves toward perceptually realistic samples. Thus, $N$ provides test-time control over the distortion--perception trade-off using a single trained model.

\textbf{Iteration depth $K$.}
Table~\ref{tab:ablation_K} evaluates $K\in\{3,5,7\}$ on random inpainting with $N{=}2$. Increasing $K$ improves PSNR, while LPIPS is minimized at $K{=}5$. Since inference cost scales as $O(NK)$, we use $K{=}5$ throughout, balancing reconstruction quality with approximately $70\%$ of the cost of $K{=}7$.

\begin{table}[t]
\centering
\small
\setlength{\tabcolsep}{4pt}
\renewcommand{\arraystretch}{0.9}
\begin{minipage}[t]{0.44\textwidth}
\centering
\caption{Iteration depth $K$}
\label{tab:ablation_K}
\begin{tabular}{ccccc}
\toprule
$K$ & PSNR $\uparrow$ & SSIM $\uparrow$ & LPIPS $\downarrow$ & ms $\downarrow$ \\
\midrule
3 & 34.63 & 0.959 & 0.020 & 137\\
5 & 35.03 &  0.963 &  0.018 & 223 \\
7 & 35.21 & 0.964 & 0.019 & 322 \\
\bottomrule
\end{tabular}
\end{minipage}
\hfill
\begin{minipage}[t]{0.55\textwidth}
\centering
\caption{Noise-level mismatch.}
\label{tab:ablation_noise}
\resizebox{\linewidth}{!}{%
\begin{tabular}{lcccccc}
\toprule
\multirow{2}{*}{Method} & \multicolumn{3}{c}{$\sigma_{\text{test}}{=}0.005$} & \multicolumn{3}{c}{$\sigma_{\text{test}}{=}0.02$} \\
\cmidrule(lr){2-4}\cmidrule(lr){5-7}
& PSNR & SSIM & LPIPS & PSNR & SSIM & LPIPS \\
\midrule
Flower5-OT & 34.02 & 0.954 & 0.019 & 33.65 & 0.947 & 0.024 \\
Ours & 35.19 & 0.965 & 0.018 & 34.57 & 0.954 & 0.020 \\
\bottomrule
\end{tabular}}
\end{minipage}
\end{table}

\textbf{Robustness to noise-level and operator mismatch.}
Training with a fixed forward model ties the learned field to the operator $\Abm$ and noise level $\sigma$. We evaluate robustness to this mismatch in Tables~\ref{tab:ablation_operator} and~\ref{tab:ablation_noise}. Table~\ref{tab:ablation_operator} applies a model trained with an inpainting mask ratio of $0.7$ to ratios ranging from $0.5$ to $0.8$. At test time, the data-consistency step uses the actual test mask, while the learned regularizer and step-size parameters remain fixed at their $p{=}0.7$ values. Performance changes smoothly across this range and improves as fewer pixels are masked, indicating graceful transfer rather than failure. Because we do not train a matched model at each ratio, this experiment measures robustness but does not quantify the mismatch penalty. Table~\ref{tab:ablation_noise} evaluates a model trained at $\sigma{=}0.01$ on other noise levels. Despite this mismatch, our model outperforms the operator-agnostic Flower5-OT on all three metrics at both tested noise levels.

\par
\begin{wraptable}[13]{R}{0.48\columnwidth}
\vspace{-2em}
\centering
\small
\setlength{\tabcolsep}{2pt}
\renewcommand{\arraystretch}{0.95}
\caption{Sensitivity to operator mismatch in random inpainting. A model
trained at mask ratio $p{=}0.7$ ($\dagger$) is evaluated at different
test-time mask ratios without retraining. Although reconstruction quality
decreases as more pixels are removed, performance does not exhibit abrupt
degradation under mismatch.}
\label{tab:ablation_operator}
\begin{tabular}{lccc}
\toprule
$p$ & PSNR $\uparrow$ & SSIM $\uparrow$ & LPIPS $\downarrow$ \\
\midrule
0.5 & 38.26 & 0.978 & 0.008 \\
0.6 & 36.95 & 0.973 & 0.012 \\
0.7$^\dagger$ & 35.03 & 0.963 & 0.018 \\
0.8 & 32.10 & 0.936 & 0.034 \\
\bottomrule
\end{tabular}
\end{wraptable}

\subsection{Additional Results}
Appendix~\ref{sec_app:add_res} provides additional quantitative and qualitative evaluations. Tables~\ref{tab:repro_celeba1000} and~\ref{tab:repro_afhq_cat400} report the results on larger test sets of 400 AFHQ-Cat and 1000 CelebA images, including standard deviations. The appendix also examines the distortion--perception trade-off (Figures~\ref{fig:pd_celeba} and~\ref{fig:pd_afhq}), reconstruction trajectories (Figures~\ref{fig:celeba_inpaint_higher_N}-\ref{fig:celeba_sr_higher_N}), error maps (Figure~\ref{fig:all_error}), inference cost (Figure~\ref{fig:pareto_celeba}), operator mismatch (Figure~\ref{fig:mismatch}), internal solver iterates (Figure~\ref{fig:trajectory_all}), and posterior sample diversity (Figure~\ref{fig:posterior_samples}). Appendix~\ref{app_sec:cs} extends the evaluation to subsampled Fourier measurements, while Appendix~\ref{app:toy} studies posterior sampling and conditional-mean estimation on a toy problem.

\section{Conclusion}

We introduced FIRM, a measurement-conditional flow-matching method that incorporates explicit data-consistency updates into each evaluation of the learned velocity field. We characterized the measurement-conditioned posterior mean as the minimizer of a variational objective and proved that the velocity defined by this minimizer transports the source distribution to the posterior. Across five restoration tasks and two datasets, the method achieves leading distortion metrics at a fraction of the inference cost of competing flow-based solvers. Varying the number of sampling steps also controls the distortion-perception trade-off without retraining.

\textbf{Limitations.}
Our theoretical analysis assumes a known linear forward model with additive Gaussian noise. Extending the framework to nonlinear measurement processes and more general noise models is an important direction for future work. We also focus on linear interpolation and explicit Euler sampling; investigating alternative probability paths and numerical solvers may further improve sampling efficiency. Finally, our experiments concentrate on two-dimensional image restoration, and applications to other imaging modalities remain to be explored.
\section*{Acknowledgments}
This work was supported in part by the National Science Foundation under CAREER award \mbox{CCF-2625643} and award \mbox{CCF-2622128}.

\bibliography{utils/refs}
\bibliographystyle{utils/IEEEbib}

\appendix
\newpage
\section{Theoretical Results Proofs}\label{app:proofs}

\paragraph{Proof roadmap.}
The proof follows the same two-stage construction as the method. First, we reduce the joint observation $(\xbm_t,\ybm)$ to an equivalent Gaussian denoising problem. The Gaussian MMSE characterization of \cite{gribonval2013reconciling} then gives the variational form of the posterior mean. Second, we show that this mean determines the conditional velocity, that the velocity satisfies the conditional continuity equation, and that its ODE transports the Gaussian source toward the posterior. The lemmas below verify these steps separately, and the final proof combines
them. 

\subsection{Variational characterization}

The Gaussian MMSE characterization connects posterior-mean estimation to
variational optimization. For a signal observed in Gaussian noise, the
posterior mean is the unique minimizer of a quadratic data-fitting term plus
an implicit regularizer determined by the signal distribution and noise
covariance. We will use this result to express the measurement-conditioned
posterior mean as the solution of an optimization problem.

\begin{lemma}[Gaussian MMSE characterization {\cite[Cor.~1]{gribonval2013reconciling}}]
\label{lem:mmse_characterization}
Let $\zbm$ have a non-degenerate distribution on $\R^n$, and let
\begin{equation}
    \ubm=\zbm+\bm{\xi},
    \qquad
    \bm{\xi}\sim\Ncal(\bm 0,\Sigmabm),
    \qquad
    \Sigmabm\succ0,
\end{equation}
where $\bm{\xi}$ is independent of $\zbm$. Then there exists a function
$\phi:\R^n\rightarrow\R\cup\{+\infty\}$ such that
\begin{equation}
    \E[\zbm\mid\ubm]
    =
    \argmin_{\zbm'\in\R^n}
    \left\{
        \frac{1}{2}
        \|\ubm-\zbm'\|_{\Sigmabm^{-1}}^2
        + \phi(\zbm')
    \right\}.
    \label{eq:gm_characterization}
\end{equation}
The minimizer is the unique global minimizer and stationary point, and
$\phi$ is unique up to an additive constant.
\end{lemma}

\begin{proof}
This is Corollary~1 of~\cite{gribonval2013reconciling}, restated using
$\|\abm\|_{\mathbf B}^2=\abm^\top\mathbf B\abm$. We use only the existence and
uniqueness conclusions in~\eqref{eq:gm_characterization}.
\end{proof}

The interpolation state $\xbm_t$ and measurement $\ybm$ are both noisy
linear observations of the clean image $\xbm_1$. Because their noise terms
are independent, the two observations can be stacked into a single linear
Gaussian model with block-diagonal noise covariance.

\begin{lemma}
\label{lem:augmented_observation}
For $t\in(0,1)$, define
\begin{equation}
    \widetilde{\ybm}
    =
    \begin{bmatrix}
        \xbm_t\\
        \ybm
    \end{bmatrix},
    \qquad
    \widetilde{\Abm}
    =
    \begin{bmatrix}
        t\Ibm\\
        \Abm
    \end{bmatrix},
    \qquad
    \widetilde{\Sigmabm}
    =
    \begin{bmatrix}
        (1-t)^2\Ibm & \bm 0\\
        \bm 0 & \sigma^2\Ibm
    \end{bmatrix}.
    \label{eq:proof_augmented_terms}
\end{equation}
Then
\begin{equation}
    \widetilde{\ybm}
    =
    \widetilde{\Abm}\xbm_1+\widetilde{\bbm},
    \qquad
    \widetilde{\bbm}
    \sim
    \Ncal(\bm 0,\widetilde{\Sigmabm}),
    \label{eq:proof_augmented_model}
\end{equation}
and $\widetilde{\bbm}$ is independent of $\xbm_1$.
\end{lemma}

\begin{proof}
The interpolation equation gives $\xbm_t=t\xbm_1+(1-t)\xbm_0$,
and the measurement equation gives $\ybm=\Abm\xbm_1+\etabm$. Therefore,
\begin{equation}
    \widetilde{\bbm}
    =
    \begin{bmatrix}
        (1-t)\xbm_0\\
        \etabm
    \end{bmatrix}.
\end{equation}
Because $\xbm_0\sim\Ncal(\bm 0,\Ibm)$ and
$\etabm\sim\Ncal(\bm 0,\sigma^2\Ibm)$ are independent of each other and of
$\xbm_1$, the stated covariance and independence follow.
\end{proof}

The augmented model is linear Gaussian, but it is not yet in the additive
denoising form required by Lemma~\ref{lem:mmse_characterization}. We combine
the interpolation state and measurement according to their noise levels to
form $\ubm_t$. The resulting variable is equivalent to observing $\xbm_1$
directly under additive Gaussian noise, with precision matrix $\Mbm_t$.
\begin{lemma}
\label{lem:reduction}
Let
\begin{equation}
    \Mbm_t
    =
    \widetilde{\Abm}^{\top}
    \widetilde{\Sigmabm}^{-1}
    \widetilde{\Abm}
    =
    \frac{t^2}{(1-t)^2}\Ibm
    +
    \frac{1}{\sigma^2}\Abm^\top\Abm
    \label{eq:proof_precision}
\end{equation}
and define
\begin{equation}
    \ubm_t
    =
    \Mbm_t^{-1}
    \widetilde{\Abm}^{\top}
    \widetilde{\Sigmabm}^{-1}
    \widetilde{\ybm}.
    \label{eq:proof_sufficient_statistic}
\end{equation}
Then $\Mbm_t\succ0$ and
\begin{equation}
    \ubm_t=\xbm_1+\bm{\xi}_t,
    \qquad
    \bm{\xi}_t\sim\Ncal(\bm 0,\Mbm_t^{-1}),
    \label{eq:proof_denoising_model}
\end{equation}
where $\bm{\xi}_t$ is independent of $\xbm_1$.
\end{lemma}

\begin{proof}
For every nonzero $\abm\in\R^n$,
\begin{equation}
    \abm^\top\Mbm_t\abm
    =
    \frac{t^2}{(1-t)^2}\|\abm\|_2^2
    +
    \frac{1}{\sigma^2}\|\Abm\abm\|_2^2
    >0.
\end{equation}
Hence $\Mbm_t\succ0$. Substituting
\eqref{eq:proof_augmented_model} into
\eqref{eq:proof_sufficient_statistic} gives
\begin{equation}
    \ubm_t
    =
    \xbm_1+
    \Mbm_t^{-1}
    \widetilde{\Abm}^{\top}
    \widetilde{\Sigmabm}^{-1}
    \widetilde{\bbm}
    =:
    \xbm_1+\bm{\xi}_t.
\end{equation}
The noise $\bm{\xi}_t$ is Gaussian and independent of $\xbm_1$. Its
covariance is
\begin{align}
    \operatorname{Cov}(\bm{\xi}_t)
    &=
    \Mbm_t^{-1}
    \widetilde{\Abm}^{\top}
    \widetilde{\Sigmabm}^{-1}
    \widetilde{\Sigmabm}
    \widetilde{\Sigmabm}^{-1}
    \widetilde{\Abm}
    \Mbm_t^{-1}=
    \Mbm_t^{-1}.
\end{align}
This proves~\eqref{eq:proof_denoising_model}.
\end{proof}

The reduced observation $\ubm_t$ is useful only if it preserves all the
information about $\xbm_1$ contained in $(\xbm_t,\ybm)$. The following
lemma shows that the augmented data-fitting term equals a Gaussian denoising
term involving $\ubm_t$, up to a constant independent of the unknown image.
Consequently, conditioning on $\ubm_t$ gives the same posterior mean as
conditioning on the original observations.
\begin{lemma}
\label{lem:sufficiency}
The statistic $\ubm_t$ in~\eqref{eq:proof_sufficient_statistic} is sufficient
for $\xbm_1$ in $\widetilde{\ybm}$. Consequently,
\begin{equation}
    \E[\xbm_1\mid\xbm_t,\ybm]
    =
    \E[\xbm_1\mid\widetilde{\ybm}]
    =
    \E[\xbm_1\mid\ubm_t].
    \label{eq:proof_equal_means}
\end{equation}
Moreover, for every $\xbm\in\R^n$,
\begin{equation}
    \frac{1}{2(1-t)^2}\|\xbm_t-t\xbm\|_2^2
    +
    \frac{1}{2\sigma^2}\|\ybm-\Abm\xbm\|_2^2
    \notag
    =
    \frac{1}{2}\|\ubm_t-\xbm\|_{\Mbm_t}^2
    +c(\widetilde{\ybm}),
    \label{eq:proof_quadratic_reduction}
\end{equation}
where $c(\widetilde{\ybm})$ does not depend on $\xbm$.
\end{lemma}

\begin{proof}
By the definitions of $\Mbm_t$ and $\ubm_t$,
\begin{equation}
    \widetilde{\Abm}^{\top}
    \widetilde{\Sigmabm}^{-1}
    \widetilde{\ybm}
    =
    \Mbm_t\ubm_t.
\end{equation}
Expanding the augmented quadratic and completing the square therefore gives
\begin{equation}
    \frac{1}{2}
    \|\widetilde{\ybm}-\widetilde{\Abm}\xbm
    \|_{\widetilde{\Sigmabm}^{-1}}^2
    \notag =
    \frac{1}{2}\xbm^\top\Mbm_t\xbm
    -
    \xbm^\top\Mbm_t\ubm_t
    +
    \frac{1}{2}
    \widetilde{\ybm}^{\top}
    \widetilde{\Sigmabm}^{-1}
    \widetilde{\ybm}=
    \frac{1}{2}\|\ubm_t-\xbm\|_{\Mbm_t}^2
    + c(\widetilde{\ybm}),
\end{equation}
where
\begin{equation}
    c(\widetilde{\ybm})
    =
    \frac{1}{2}
    \widetilde{\ybm}^{\top}
    \widetilde{\Sigmabm}^{-1}
    \widetilde{\ybm}
    -
    \frac{1}{2}\ubm_t^\top\Mbm_t\ubm_t.
\end{equation}
Expanding the left-hand side blockwise yields
\eqref{eq:proof_quadratic_reduction}.

The conditional likelihood consequently factors as
\begin{equation}
    p(\widetilde{\ybm}\mid\xbm_1)
    =
    r(\widetilde{\ybm})\,
    \exp\left(
        -\frac{1}{2}
        \|\ubm_t-\xbm_1\|_{\Mbm_t}^2
    \right),
\end{equation}
where $r(\widetilde{\ybm})$ is independent of $\xbm_1$. The factorization
criterion shows that $\ubm_t$ is sufficient for $\xbm_1$, which proves
\eqref{eq:proof_equal_means}.
\end{proof}

The preceding lemmas reduce the observations $(\xbm_t,\ybm)$ to an
equivalent Gaussian denoising observation $\ubm_t$ without losing information
about $\xbm_1$. We can therefore apply the Gaussian MMSE characterization to
$\ubm_t$ and then rewrite its quadratic term using the original observations.
The resulting objective contains one term associated with the interpolation
state and another that explicitly enforces the measurement model.

\begin{lemma}
\label{lem:posterior_variational}
Suppose that $p_1$ is non-degenerate. For every $t\in(0,1)$, there
exists a function
$\phi_t:\R^n\rightarrow\R\cup\{+\infty\}$, unique up to an additive
constant, such that
\begin{align}
    \E[\xbm_1\mid\xbm_t,\ybm]
    =
    \argmin_{\xbm\in\R^n}
    \bigg\{
        &\frac{1}{2(1-t)^2}\|\xbm_t-t\xbm\|_2^2
        +
        \frac{1}{2\sigma^2}\|\ybm-\Abm\xbm\|_2^2
        +
        \phi_t(\xbm)
    \bigg\}.
    \label{eq:proof_variational_result}
\end{align}
The posterior mean is the unique global minimizer and stationary point.
The dependence of $\phi_t$ on the fixed operator $\Abm$ and noise level
$\sigma$ is suppressed for readability.
\end{lemma}

\begin{proof}
Lemma~\ref{lem:reduction} expresses $\ubm_t$ as the Gaussian denoising model
\begin{equation}
    \ubm_t=\xbm_1+\bm{\xi}_t,
    \qquad
    \bm{\xi}_t\sim\Ncal(\bm 0,\Mbm_t^{-1}).
\end{equation}
Applying Lemma~\ref{lem:mmse_characterization} with noise covariance
$\Mbm_t^{-1}$ gives a function $\phi_t$ for which
\begin{equation}
    \E[\xbm_1\mid\ubm_t]
    =
    \argmin_{\xbm\in\R^n}
    \left\{
        \frac{1}{2}\|\ubm_t-\xbm\|_{\Mbm_t}^2
        +
        \phi_t(\xbm)
    \right\}.
\end{equation}
Lemma~\ref{lem:sufficiency} shows that the conditional mean equals
$\E[\xbm_1\mid\xbm_t,\ybm]$ and that the quadratic term differs from the two
data terms in~\eqref{eq:proof_variational_result} only by a constant
independent of $\xbm$. Adding or removing that constant does not change the
minimizer or stationary points, proving the result.
\end{proof}

\subsection{Conditional probability flow}

The variational result characterizes the posterior mean, whereas flow
matching requires a velocity field. Under linear interpolation, the velocity
of each trajectory can be written using only its endpoint $\xbm_1$ and
current state $\xbm_t$. Averaging over the possible endpoints conditioned on
$\xbm_t$ and $\ybm$ therefore expresses the optimal conditional velocity
directly in terms of the posterior mean.

\begin{lemma}
\label{lem:velocity_identity}
For $t\in[0,1)$, the optimal measurement-conditional velocity is
\begin{equation}
    \vbm(\xbm,t,\ybm)
    =
    \E[\xbm_1-\xbm_0\mid\xbm_t=\xbm,\ybm]
    =
    \frac{
        \E[\xbm_1\mid\xbm_t=\xbm,\ybm]-\xbm
    }{
        1-t
    }.
    \label{eq:proof_velocity_identity}
\end{equation}
\end{lemma}

\begin{proof}
The linear path has velocity
$\dot{\xbm}_t=\xbm_1-\xbm_0$. Since
\begin{equation}
    \xbm_0=\frac{\xbm_t-t\xbm_1}{1-t},
\end{equation}
we obtain
\begin{equation}
    \xbm_1-\xbm_0
    =
    \frac{\xbm_1-\xbm_t}{1-t}.
\end{equation}
Taking the conditional expectation given $\xbm_t=\xbm$ and $\ybm$ proves
\eqref{eq:proof_velocity_identity}.
\end{proof}

To establish a probability flow, we must first show that the posterior mean
changes smoothly with the current state $\xbm$. The bounded-support assumption
allows differentiation through the conditional density. The resulting
Jacobian has a simple covariance form, which will be used to control the
regularity of the conditional velocity.
\begin{lemma}[Regularity of the conditional mean]
\label{lem:mean_regularity}
Suppose
$\operatorname{supp}p_1\subseteq\overline{B}(\bm 0,R)$ and fix
$\ybm$. Define
\begin{equation}
    \mbm_t(\xbm)
    =
    \E[\xbm_1\mid\xbm_t=\xbm,\ybm].
\end{equation}
For every $t\in[0,1)$, the conditional density
$p_t(\xbm\mid\ybm)$ is positive and smooth, and $\mbm_t$ is smooth in
$(\xbm,t)$. Moreover, $\|\mbm_t(\xbm)\|_2\leq R$ and
\begin{equation}
    \nabla_{\xbm}\mbm_t(\xbm)
    =
    \frac{t}{(1-t)^2}
    \operatorname{Cov}
    (\xbm_1\mid\xbm_t=\xbm,\ybm).
    \label{eq:proof_mean_jacobian}
\end{equation}
\end{lemma}

\begin{proof}
Let $\nu=p(\xbm_1\mid\ybm)$. Because the Gaussian likelihood is positive,
$\nu$ is supported on $\overline{B}(\bm 0,R)$. For $t<1$, the conditional
density of $\xbm_t$ is
\begin{equation}
    p_t(\xbm\mid\ybm)
    =
    \int
    \varphi_{1-t}(\xbm-t\xbm_1)
    \,\nu(\mathrm d\xbm_1),
    \label{eq:proof_conditional_density}
\end{equation}
where $\varphi_s$ denotes the density of
$\Ncal(\bm 0,s^2\Ibm)$. The Gaussian kernel is positive and smooth.
Because $\nu$ has bounded support, its derivatives are dominated uniformly
on compact subsets of $\R^n\times[0,1)$. We may therefore differentiate
under the integral sign.

The conditional distribution is supported on
$\overline{B}(\bm 0,R)$, so its mean satisfies
$\|\mbm_t(\xbm)\|_2\leq R$.

Bayes' rule gives
\begin{equation}
    q_{\xbm,t}(\mathrm d\xbm_1)
    =
    \frac{
        \varphi_{1-t}(\xbm-t\xbm_1)
    }{
        p_t(\xbm\mid\ybm)
    }
    \nu(\mathrm d\xbm_1).
\end{equation}
Differentiating its log-density with respect to $\xbm$ yields
\begin{equation}
    \nabla_{\xbm}\log q_{\xbm,t}(\xbm_1)
    =
    \frac{t}{(1-t)^2}
    \bigl(\xbm_1-\mbm_t(\xbm)\bigr).
\end{equation}
Since
$\mbm_t(\xbm)=\int\xbm_1q_{\xbm,t}(\mathrm d\xbm_1)$, differentiation
under the integral gives~\eqref{eq:proof_mean_jacobian}.
\end{proof}

A unique probability flow requires the conditional velocity to vary
regularly with the current state. The Jacobian identity from
Lemma~\ref{lem:mean_regularity} allows us to bound this variation through the
conditional covariance of $\xbm_1$. Bounded support provides a uniform
covariance bound, making the velocity globally Lipschitz on every time
interval $[0,\tau]$ with $\tau<1$.

\begin{lemma}[Lipschitz conditional velocity]
\label{lem:regularity}
Under the assumptions of Lemma~\ref{lem:mean_regularity},
$\vbm(\cdot,t,\ybm)$ is globally Lipschitz for every $t\in[0,1)$, with
\begin{equation}
    \operatorname{Lip}
    (\vbm(\cdot,t,\ybm))
    \leq
    \frac{1}{1-t}
    +
    \frac{tR^2}{(1-t)^3}.
    \label{eq:proof_lipschitz_bound}
\end{equation}
The bound is uniform over $t\in[0,\tau]$ for every $\tau<1$.
\end{lemma}

\begin{proof}

Using Lemma~\ref{lem:velocity_identity} and Lemma~\ref{lem:mean_regularity},
\begin{equation}
    \nabla_{\xbm}\vbm(\xbm,t,\ybm)
    =
    \frac{t}{(1-t)^3}
    \operatorname{Cov}
    (\xbm_1\mid\xbm_t=\xbm,\ybm)
    -
    \frac{1}{1-t}\Ibm.
\end{equation}
For every unit vector $\abm$,
\begin{equation}
    \abm^\top
    \operatorname{Cov}
    (\xbm_1\mid\xbm_t=\xbm,\ybm)
    \abm
    \leq
    \E[(\abm^\top\xbm_1)^2\mid\xbm_t=\xbm,\ybm]
    \leq R^2.
\end{equation}
Thus the conditional covariance has operator norm at most $R^2$, giving
\eqref{eq:proof_lipschitz_bound}. For every $\tau<1$,
\begin{equation}
    \sup_{t\in[0,\tau]}
    \operatorname{Lip}(\vbm(\cdot,t,\ybm))
    \leq
    \frac{1}{1-\tau}
    +
    \frac{\tau R^2}{(1-\tau)^3}
    <\infty.
\end{equation}
At $t=0$, the identity
$\vbm(\xbm,0,\ybm)=\E[\xbm_1\mid\ybm]-\xbm$ gives Lipschitz constant
one, consistent with the bound.
\end{proof}

For each possible endpoint $\xbm_1$, linear interpolation defines a simple
Gaussian probability path with a known velocity. The conditional distribution
$p_t(\xbm\mid\ybm)$ is obtained by averaging these paths over the posterior
distribution of $\xbm_1$. Averaging their probability fluxes yields the
optimal measurement-conditional velocity and establishes the corresponding
continuity equation.
\begin{lemma}[Conditional continuity equation]
\label{lem:conditional_continuity}
Under the assumptions of Lemma~\ref{lem:regularity}, the conditional density
$p_t(\xbm\mid\ybm)$ satisfies
\begin{equation}
    \partial_t p_t(\xbm\mid\ybm)
    +
    \nabla\cdot
    \left(
        p_t(\xbm\mid\ybm)
        \vbm(\xbm,t,\ybm)
    \right)
    =
    0,
    \qquad t\in[0,1).
    \label{eq:proof_conditional_continuity}
\end{equation}
\end{lemma}

\begin{proof}
Fix $\ybm$ and write $\nu=p(\xbm_1\mid\ybm)$. For a fixed endpoint
$\xbm_1$, the density
\begin{equation}
    p_t(\xbm\mid\xbm_1)
    =
    \varphi_{1-t}(\xbm-t\xbm_1)
\end{equation}
is induced by the deterministic map
$\xbm_0\mapsto t\xbm_1+(1-t)\xbm_0$. It therefore satisfies
\begin{equation}
    \partial_t p_t(\xbm\mid\xbm_1)
    +
    \nabla\cdot
    \left(
        p_t(\xbm\mid\xbm_1)
        \frac{\xbm_1-\xbm}{1-t}
    \right)
    =
    0.
    \label{eq:proof_endpoint_continuity}
\end{equation}
Because $\xbm_0$ is independent of $(\xbm_1,\ybm)$,
\begin{equation}
    p_t(\xbm\mid\ybm)
    =
    \int p_t(\xbm\mid\xbm_1)\,\nu(\mathrm d\xbm_1).
\end{equation}
Integrating~\eqref{eq:proof_endpoint_continuity} with respect to $\nu$ gives
\begin{align}
    \partial_t p_t(\xbm\mid\ybm)
    =
    -\nabla\cdot
    \int
    p_t(\xbm\mid\xbm_1)
    \frac{\xbm_1-\xbm}{1-t}
    \,\nu(\mathrm d\xbm_1).
    \label{eq:proof_marginalized_flux}
\end{align}
By Bayes' rule and Lemma~\ref{lem:velocity_identity}, the integral in
\eqref{eq:proof_marginalized_flux} equals
\begin{equation}
    p_t(\xbm\mid\ybm)
    \E\left[
        \frac{\xbm_1-\xbm}{1-t}
        \,\middle|\,
        \xbm_t=\xbm,\ybm
    \right]
    =
    p_t(\xbm\mid\ybm)\vbm(\xbm,t,\ybm).
\end{equation}
Substitution proves~\eqref{eq:proof_conditional_continuity}.
\end{proof}

The continuity equation identifies the probability path associated with the
conditional velocity, but a well-defined sampler also requires the
corresponding ODE to have unique trajectories. The Lipschitz bound guarantees a unique flow on every interval that ends before $t=1$. Because the ODE flow and the conditional path start from the same Gaussian distribution and satisfy the same continuity equation, their distributions agree. The conditional path then converges to the measurement-conditioned posterior as  $t$ approaches one.

\begin{lemma}[Conditional transport and posterior limit]
\label{lem:posterior_transport}
Under the assumptions of Lemma~\ref{lem:regularity}, for every $\tau<1$ the
ODE
\begin{equation}
    \frac{\dd\xbm_t}{\dd t}
    =
    \vbm(\xbm_t,t,\ybm)
    \label{eq:proof_conditional_ode}
\end{equation}
generates a unique flow map $\Phi_{0\rightarrow\tau}$ satisfying
\begin{equation}
    (\Phi_{0\rightarrow\tau})_{\#}
    \Ncal(\bm 0,\Ibm)
    =
    p_\tau(\cdot\mid\ybm).
    \label{eq:proof_pushforward}
\end{equation}
Furthermore,
\begin{equation}
    p_\tau(\cdot\mid\ybm)
    \Rightarrow
    p(\xbm_1\mid\ybm)
    \qquad\text{as }\tau\rightarrow1^-.
    \label{eq:proof_weak_limit}
\end{equation}
\end{lemma}

\begin{proof}
Lemma~\ref{lem:mean_regularity} and Lemma~\ref{lem:regularity} show that the velocity is continuous in time and
globally Lipschitz in $\xbm$, uniformly on every interval $[0,\tau]$ with
$\tau<1$. It also has at most linear growth in $\xbm$. The
Picard--Lindel\"of theorem therefore gives a unique global flow map on each
such interval.

The pushforward of the initial distribution through this flow satisfies the
continuity equation associated with $\vbm$. By
Lemma~\ref{lem:conditional_continuity}, $p_t(\cdot\mid\ybm)$ satisfies the
same equation. The two solutions also have the same initial distribution
because $\xbm_0$ is independent of $\ybm$:
\begin{equation}
    p_0(\cdot\mid\ybm)
    =
    \Ncal(\bm 0,\Ibm).
\end{equation}
Uniqueness of the transported measure under a globally Lipschitz velocity
proves~\eqref{eq:proof_pushforward}.

Finally, couple the conditional path by drawing
$\xbm_1\sim p(\xbm_1\mid\ybm)$ and
$\xbm_0\sim\Ncal(\bm 0,\Ibm)$ independently. Then
\begin{equation}
    \xbm_\tau
    =
    \tau\xbm_1+(1-\tau)\xbm_0
    \longrightarrow
    \xbm_1
    \qquad\text{almost surely}.
\end{equation}
Almost-sure convergence implies convergence in distribution, proving
\eqref{eq:proof_weak_limit}. The result requires no extension of the flow
map to $t=1$.
\end{proof}

\subsection{Proof of Proposition~\ref{prop:conditional_flow}}
\label{ap_sec:proof_pros1}

\begin{proof}
Lemma~\ref{lem:posterior_variational} proves that the
measurement-conditioned posterior mean is the unique minimizer of
\eqref{eq:mmse_variational}. Lemma~\ref{lem:velocity_identity} converts this
mean into the optimal conditional velocity. Finally,
Lemmas~\ref{lem:conditional_continuity} and~\ref{lem:posterior_transport}
show that the resulting velocity defines the conditional probability flow
and converges weakly to the measurement-conditioned posterior. These are
the two claims of Proposition~\ref{prop:conditional_flow}.
\end{proof}

\section{Algorithm Details}\label{app:alg_details}

This section provides implementation details of the HQS iterations used to parameterize the conditional velocity. The main text gives the HQS objective and its two alternating subproblems in~\eqref{eq:x_update} and~\eqref{eq:z_update}.

\subsection{Data-Consistency Update}\label{app:x_update}
The $\xbm$-subproblem is
\begin{equation}
    \xbm^{k+1}
    =
    \argmin_{\xbm}
    \left\{
        \frac{1}{2\sigma^2}
        \|\ybm-\Abm\xbm\|_2^2
        +
        \frac{\mu}{2}
        \|\xbm-\zbm^k\|_2^2
    \right\}.
    \label{eq:appendix_x_update}
\end{equation}
Its optimality condition is
\begin{equation}
    \left(\sigma^{-2}\Abm^\top\Abm+\mu\Ibm\right)\xbm^{k+1}
    =
    \sigma^{-2}\Abm^\top\ybm+\mu\zbm^k.
\end{equation}
Defining $\rho=\mu\sigma^2$ gives the normal equations used by the
implementation,
\begin{equation}
    \left(\Abm^\top\Abm+\rho\Ibm\right)\xbm^{k+1}
    =
    \Abm^\top\ybm+\rho\zbm^k.
    \label{eq:normal_equations}
\end{equation}
The update depends explicitly on the forward operator $\Abm$. Closed-form
solutions for the operators considered in this work are given in
Appendix~\ref{app:operator_solutions}.

\subsection{Learned Regularization Update}

The exact $\zbm$-subproblem is
\begin{equation}
    \zbm_\star^{k+1}
    =
    \argmin_{\zbm}
    \left\{
        \phi_t(\zbm)
        +
        \frac{1}{2(1-t)^2}
        \|\xbm_t-t\zbm\|_2^2
        +
        \frac{\mu}{2}
        \|\xbm^{k+1}-\zbm\|_2^2
    \right\}.
    \label{eq:appendix_z_update}
\end{equation}
The two quadratic terms can be written as the weighted residual
\begin{equation}
    \frac{1}{2}
    \left\|
        \widetilde{\ybm}_k-\widetilde{\Abm}_k\zbm
    \right\|_{\widetilde{\Sigmabm}_k^{-1}}^2,
\end{equation}
where
\begin{equation}
    \widetilde{\ybm}_k
    =
    \begin{bmatrix}
        \xbm_t\\
        \xbm^{k+1}
    \end{bmatrix},
    \qquad
    \widetilde{\Abm}_k
    =
    \begin{bmatrix}
        t\Ibm\\
        \Ibm
    \end{bmatrix},
    \qquad
    \widetilde{\Sigmabm}_k
    =
    \begin{bmatrix}
        (1-t)^2\Ibm & \zerobm\\
        \zerobm & \mu^{-1}\Ibm
    \end{bmatrix}.
    \label{eq:z_augmented_terms}
\end{equation}
Completing the square reduces the two observations to the sufficient statistic
\begin{equation}
    \ubm_k
    =
    \frac{
        \dfrac{t}{(1-t)^2}\xbm_t
        +
        \mu\xbm^{k+1}
    }{
        \dfrac{t^2}{(1-t)^2}+\mu
    },
    \label{eq:z_sufficient_statistic}
\end{equation}
with precision
\begin{equation}
    \Mbm_k
    =
    \left(
        \frac{t^2}{(1-t)^2}+\mu
    \right)\Ibm.
\end{equation}
Consequently, ~\eqref{eq:appendix_z_update} is equivalent, up to terms
independent of $\zbm$, to
\begin{equation}
    \zbm_\star^{k+1}
    =
    \argmin_{\zbm}
    \left\{
        \phi_t(\zbm)
        +
        \frac{1}{2}
        \|\ubm_k-\zbm\|_{\Mbm_k}^2
    \right\}.
    \label{eq:z_reduced_update}
\end{equation}
Thus, the exact update denoises a precision-weighted combination of the
interpolation state and the data-consistent iterate. The reduction identifies
the form of the regularization subproblem but does not imply that an arbitrary
minimizer of~\eqref{eq:z_reduced_update} is a posterior mean. In particular,
the converse direction of the MMSE characterization in
\cite{gribonval2013reconciling} does not hold in general.

Because $\phi_t$ is unavailable, we follow PnP
practice~\cite{venkatakrishnan2013plug} and replace the exact update with a
learned estimator,
\begin{equation}
    \widehat{\zbm}^{k+1}
    =
    \Rsf_{\thetabm}
    \left(
        \xbm^{k+1},
        \xbm_t,
        t
    \right).
    \label{eq:learned_z_update}
\end{equation}
The network receives the two components of~\eqref{eq:z_sufficient_statistic}
separately, allowing their relative weighting to depend on $t$ and the HQS
parameters.

Following~\cite{ryu2019plug}, we damp the learned estimate toward the
data-consistent iterate,
\begin{equation}
    \zbm^{k+1}
    =
    (1-\beta)\xbm^{k+1}
    +
    \beta\widehat{\zbm}^{k+1},
    \qquad
    \beta\in[0,1].
    \label{eq:z_update_final}
\end{equation}
Damping retains an explicit contribution from $\xbm^{k+1}$ and prevents the
learned estimator from completely overwriting the data-consistent iterate.

\subsection{Velocity Parameterization and Cost}

Starting from $\zbm^0=\xbm_t$, we alternate \eqref{eq:normal_equations}, \eqref{eq:learned_z_update},
and~\eqref{eq:z_update_final} for $K$ iterations. The final iterate estimates
the measurement-conditioned posterior mean,
\begin{equation}
    \zbm_{\thetabm}^K(\xbm_t,t,\ybm)
    \approx
    \E[\xbm_1\mid\xbm_t,\ybm],
\end{equation}
and defines the learned conditional velocity through~\eqref{eq:induced_velocity}.

The network parameters $\thetabm$ enter only through
$\Rsf_{\thetabm}$. Each velocity evaluation requires $K$ data-consistency
solves and $K$ learned-estimator evaluations. Since every data-consistency
solve depends on $\Abm$ and $\ybm$, the forward operator acts within every
inner iteration of the learned velocity field.

Algorithm~\ref{alg:posterior_mean} summarizes the operator-aware posterior mean evaluation.

\begin{algorithm}[t]
    \caption{Operator-aware posterior-mean estimation}
    \label{alg:posterior_mean}
    \begin{algorithmic}[1]
        \Require Interpolation state $\xbm_t$, time $t<1$, measurement $\ybm$,
        forward operator $\Abm$, learned estimator $\Rsf_{\thetabm}$,
        iteration count $K$, coupling $\rho$, and damping $\beta$
        \Ensure Posterior-mean estimate $\zbm^K$
        \State $\zbm^0 \gets \xbm_t$
        \For{$k=0,\ldots,K-1$}
            \State $\xbm^{k+1} \gets
            \left(\Abm^\top\Abm+\rho\Ibm\right)^{-1}
            \left(\Abm^\top\ybm+\rho\zbm^k\right)$
            \State $\widehat{\zbm}^{k+1} \gets
            \Rsf_{\thetabm}(\xbm^{k+1},\xbm_t,t)$
            \State $\zbm^{k+1} \gets
            (1-\beta)\xbm^{k+1}+\beta\widehat{\zbm}^{k+1}$
        \EndFor
        \State \Return $\zbm^K$
    \end{algorithmic}
\end{algorithm}

\subsection{Operator-Specific Data-Consistency Solutions}
\label{app:operator_solutions}

\paragraph{Image denoising.}
For denoising, $\Abm=\Ibm$, and ~\eqref{eq:normal_equations} gives
\begin{equation}
    \xbm^{k+1}
    =
    \frac{\ybm+\rho\zbm^k}{1+\rho}.
    \label{eq:denoising_update}
\end{equation}
The update is a weighted average of the measurement and the current iterate.

\paragraph{Image inpainting.}
For inpainting, $\Abm=\Mbm$ is a diagonal binary mask satisfying
$\Mbm^\top\Mbm=\Mbm$. Equation~\eqref{eq:normal_equations} can therefore be
solved elementwise,
\begin{equation}
    x_i^{k+1}
    =
    \begin{cases}
        \dfrac{y_i+\rho z_i^k}{1+\rho}, & i\in\Omega,\\[2ex]
        z_i^k, & i\notin\Omega,
    \end{cases}
    \label{eq:inpainting_update}
\end{equation}
where $\Omega$ denotes the observed pixels.

\paragraph{Image super-resolution.}
We implement super-resolution using stride subsampling without an anti-aliasing
prefilter. The forward operator retains a periodic subset of pixels, and
$\Abm^\top\Abm$ is a diagonal binary mask. The update is therefore
~\eqref{eq:inpainting_update}, with $\ybm$ replaced by
$\Abm^\top\ybm$ and $\Omega$ denoting the sampled grid locations.

\paragraph{Image deblurring.}
For deblurring, $\Abm=\Hbm$ is a spatially invariant blur. Under circular
boundary conditions,
$\Hbm=\Fbm^H\bm{\Lambda}\Fbm$, where $\bm{\Lambda}$ contains the Fourier
coefficients of the blur kernel. Equation~\eqref{eq:normal_equations} then gives
\begin{equation}
    \xbm^{k+1}
    =
    \Fbm^H
    \left[
        \frac{
            \overline{\bm{\Lambda}}\,\Fbm\ybm
            +
            \rho\Fbm\zbm^k
        }{
            |\bm{\Lambda}|^2+\rho
        }
    \right],
    \label{eq:deblurring_update}
\end{equation}
where all operations inside the brackets are elementwise.

\paragraph{Subsampled Fourier measurements.}
For the compressed-sensing experiments in Appendix~\ref{app_sec:cs}, let
$\Abm=\Pbm\Fbm$, where $\Fbm$ is the orthonormal two-dimensional DFT and
$\Pbm=\operatorname{diag}(\bm{p})$ is the sampling mask. Writing
$\widehat{\zbm}^{,k}=\Fbm\zbm^k$, the Fourier-domain update is
\begin{equation}
    \widehat{x}_i^{,k+1}
    =
    \begin{cases}
        \dfrac{y_i+\rho\widehat{z}_i^{,k}}{1+\rho},
        & i\in\Omega,\\[2ex]
        \widehat{z}_i^{,k},
        & i\notin\Omega,
    \end{cases}
    \qquad
    \xbm^{k+1}
    =
    \operatorname{Re}
    \left\{
        \Fbm^H\widehat{\xbm}^{,k+1}
    \right\}.
    \label{eq:cs_update}
\end{equation}
For the conjugate-symmetric mask used in our experiments, taking the real part
is equivalent to solving the corresponding real normal equations. The operator
is applied independently to each color channel.

\paragraph{General linear operators.}
For operators without exploitable structure, ~\eqref{eq:normal_equations}
can be solved using conjugate gradients. All operators used in this work admit
the closed-form updates described above.

\section{Related Works}\label{app:rel_work}
 
Generative solvers for imaging inverse problems differ mainly in whether the forward operator $\Abm$ enters only at inference, through an operator-aware architecture, or during conditional training. We organize prior work along this axis.
 
\paragraph{Inference-time guidance with a pretrained prior.}
Diffusion-based methods combine an unconditional prior with measurement consistency at sampling time. DPS and $\Pi$GDM approximate likelihood guidance using denoised clean-image estimates \cite{chung2023diffusion,song2023pseudoinverseguided}; DDRM and DDNM exploit the spectral or null-space structure of $\bm A$ \cite{kawar2022denoising,wang2023ddnm}; and DiffPIR, variational PnP diffusion, and DAPS alternate generative updates with optimization, projection, or conditional resampling \cite{zhu2023denoising,mardani2024variational,zhang2025improving}. Flow-based counterparts include OT-ODE, D-Flow, Flow Priors, FlowDPS, PnP-Flow, and Flower \cite{pokle2024training,benhamu2024dflow,zhang2024flow,kim2025flowdps,martin2025pnpflow,pourya2026flower}. These methods can reuse a pretrained prior across many compatible operators, but construct only an approximation to the measurement-conditional dynamics at inference. A recent posterior-transport analysis places several such approximations in a common framework and bounds their posterior bias \cite{xu2026flow}.
 
\paragraph{Amortized conditional models.}
Conditional restoration methods instead learn the measurement-to-image dynamics from paired data. Examples include SR3, Palette, InDI, I\textsuperscript{2}SB, conditional diffusion bridges, ResShift, and ResFlow \cite{saharia2022image,saharia2022palette,delbracio2023inversion,liu2023i2sb,chung2023cddb,yue2023resshift,qin2025resflow}. Within flow matching, data-dependent stochastic interpolants and conditional Wasserstein formulations provide theoretical foundations \cite{albergostochastic,chemseddine2025conditional}, while DAWN-FM conditions its velocity field on a measurement embedding, typically $\Abm^{\top}\ybm$, and the noise level \cite{ahamed2024dawn}. Most such models are trained for a particular task or degradation family and encode the measurement through $\ybm$, $\Abm^{\top}\ybm$, or a learned representation. Our distinction is that $\bm A$ acts explicitly through data-consistency updates inside the generative dynamics.

\paragraph{Operator-aware learned solvers.}Plug-and-play methods retain the forward operator in a data-fidelity update while replacing the regularizer with a learned denoiser \cite{venkatakrishnan2013plug,zhang2021plug,hurault2022gradient}. Algorithm unrolling makes these iterations end-to-end trainable, yielding methods such as ISTA-Net, ADMM-Net, learned primal--dual, MoDL, and VarNet \cite{zhang2018ista,yang2016admmnet,adler2018learnedpd,aggarwal2019modl,sriram2020varnet}. MoDL is particularly close in its alternation between data consistency and a shared-weight learned update. These architectures are typically trained as point estimators and do not by themselves define posterior samplers. Our method imports their operator-aware structure into a generative velocity field by applying data consistency at every inner iteration of every ODE step.

\paragraph{Distortion--perception control.}The distortion--perception trade-off is fundamental in image restoration \cite{blau2018perception}. For MSE distortion and Wasserstein perception, optimal estimators trace a Wasserstein geodesic between the minimum-MSE and perfect-perception endpoints \cite{freirich2021theory}. PMRF operationalizes this idea by transporting a posterior-mean estimate toward the data distribution \cite{ohayon2025pmrf}; related interpolation mechanisms appear in InDI and data-dependent interpolants \cite{delbracio2023inversion,albergostochastic}. Flow-map denoisers similarly produce a one-parameter family of estimators that traverses the distortion--perception plane, with optimality established in the Gaussian setting \cite{zilberstein2026flowmap}. Our Euler step count $N$ plays an analogous empirical role. Under the ideal conditional velocity, one Euler step reduces to a conditional-mean update; for a learned velocity field, the $N{=}1$ output should be interpreted as posterior-mean-like rather than exactly equal to the posterior mean. Increasing $N$ evaluates the learned estimator at progressively lower interpolation noise levels, where posterior modes become more resolvable.

Our model is trained for a specified family of forward operators and noise levels,  placing it in the amortized conditional family. The central question is whether, given operator-specific training, explicitly embedding $\bm A$ in the dynamics improves upon conditioning a network only on an image-space measurement. We investigate this question in Appendix~\ref{app:hqs_ablation}. We also compare with inference-time flow solvers because they are established reference points for the benchmark \cite{pokle2024training,benhamu2024dflow,zhang2024flow,martin2025pnpflow,pourya2026flower}. Methods that amortize across families of operators \cite{terris2026ram,elata2025invfussion} provide a natural direction for extending our approach.

\section{Implementation Details} \label{app:implementation}

\paragraph{Estimator network.}
The learned estimator $\Rsf_{\thetabm}$ is a U-Net that takes the concatenation of the
data-consistent iterate $\xbm^{k+1}$ and the flow state $\xbm_t$ along the channel
dimension ($2C$ input channels) and is conditioned on $t$ through a sinusoidal
embedding of width $4\,\mathrm{ch}$ injected into every residual block. Blocks use
GroupNorm with 32 groups and Swish activations, with self-attention at the resolutions
listed in Table~\ref{tab:hparams}. Weights are shared across the $K$ inner
iterations, so the parameter count of the full model equals that of a single U-Net.
CelebA models use $\mathrm{ch}=32$ with multipliers $(1,2,4,8)$ ($34.5$M parameters);
AFHQ-Cat models use multipliers $(1,2,4,8,8)$ with $\mathrm{ch}=32$ ($59.9$M), except
deblurring which uses $\mathrm{ch}=64$ ($239.4$M).

\paragraph{Solver.}
We use $K=5$ solver inner iterations, damping $\beta=0.5$, and a fixed coupling
$\rho=0.01$ for every operator and noise level. Writing the data-consistency step as
$\xbm^{k+1}=(\Abm^{\top}\Abm+\rho\Ibm)^{-1}(\Abm^{\top}\ybm+\rho\zbm^{k})$ absorbs the
noise variance into $\rho=\mu\sigma^{2}$; the solve is closed-form for all operators we
consider (Appendix~\ref{app:alg_details}). The iteration is initialized at
$\zbm^{0}=\xbm_t$.

\paragraph{Training.}
AdamW with learning rate $10^{-4}$ decayed by a cosine schedule to $10^{-6}$, weight
decay $10^{-4}$, gradient-norm clipping at $1.0$, and no mixed precision. We train with
a total batch size of $32$ (8 per GPU on 4 GPUs, DDP) for
$500$ epochs of $2000$ steps on CelebA and $700$ epochs of $750$ steps on AFHQ-Cat,
where each epoch resamples the training set with replacement. An exponential moving
average of the weights with decay $0.999$ is maintained from epoch $20$ and used for all
validation and test-time evaluation. For the first $20$ epochs we train the estimator
without data-consistency, using a single network call on the concatenation of $\Abm^{\top}\ybm$
and $\xbm_t$ and then enable the $K=5$ data-term solver for the remainder of
training; this warm start avoids backpropagating through five iterations of an
untrained estimator. Times $t$ are drawn by low-discrepancy stratified sampling over
$64$ equal strata of $[t_{\min},t_{\max}]$ rather than i.i.d.\ uniformly, which reduces
the variance of the epoch-averaged loss at small batch size.  Training roughly takes between 24--48 hours. 

\paragraph{Evaluation protocol.}
PSNR is computed on $[0,1]$ with \texttt{data\_range}$=1$,
matching the postprocessing used by \cite{pourya2026flower} so that the numbers are directly comparable; SSIM is computed on the clamped $[0,1]$ image and LPIPS with an AlexNet backbone on the clamped $[-1,1]$ image. All runs use seed $42$.

\begin{table}[t]
\caption{Hyperparameters. All values are held fixed across the five tasks unless noted.}
\label{tab:hparams}
\centering
\small
\begin{tabular}{llcc}
\toprule
& & CelebA ($128^2$) & AFHQ-Cat ($256^2$) \\
\midrule
\multirow{4}{*}{Solver}
& iteration depth $K$                & 5 & 5 \\
& damping $\beta$                    & 0.5 & 0.5 \\
& coupling $\rho$                    & 0.01 & 0.01 \\
& initialization $\zbm^{0}$          & $\xbm_t$ & $\xbm_t$ \\
\midrule
\multirow{4}{*}{Flow}
& $t_{\min}$ / $t_{\max}$            & $10^{-3}$ / $0.995$ & $10^{-3}$ / $0.995$ \\
& clamp $\tau_{\min}$         & 0.1 & 0.1 \\
& integrator                         & Euler & Euler \\
\midrule
\multirow{5}{*}{Network}
& base width $\mathrm{ch}$           & 32 & 32 (64 for deblurring) \\
& channel multipliers                & $(1,2,4,8)$ & $(1,2,4,8,8)$ \\
& residual blocks per level          & 6 & 6 \\
& attention resolutions              & 16, 8 & 32, 16 \\
& parameters                         & 34.5\,M & 59.9\,M (239.4\,M) \\
\midrule
\multirow{6}{*}{Optimization}
& optimizer                          & AdamW & AdamW \\
& learning rate (cosine to $10^{-6}$)& $10^{-4}$ & $10^{-4}$ \\
& weight decay / grad clip           & $10^{-4}$ / 1.0 & $10^{-4}$ / 1.0 \\
& batch size (total)                 & 32 & 32 \\
& epochs $\times$ steps per epoch    & $500\times2000$ & $700\times750$ \\
& EMA decay (from epoch 20)          & 0.999 & 0.999 \\
\bottomrule
\end{tabular}
\end{table}

\subsection{Baselines}
The flow baselines share the pretrained prior of \cite{martin2025pnpflow}: $34.5$M parameters on CelebA and $31.0$M on AFHQ-Cat (computed from the released checkpoints). Our CelebA models match this capacity exactly; on AFHQ-Cat we use a larger estimator for deblurring, the hardest of our settings at $256\times256$, where smaller estimators underfit. Its parameters are shared across the $K$ inner iterations.

\paragraph{PnP-HQS (DPIR).}
PnP-HQS is DPIR~\cite{zhang2021plug}, which is a  half-quadratic splitting with a DRUNet
denoiser prior, alternating the closed-form data step
$(\Abm^\top\Abm+\rho_k\Ibm)\xbm=\Abm^\top\ybm+\rho_k\zbm$ with a denoiser step at
noise level $\sigma_k$, log-spaced from $49/255$ down to the measurement noise, with
$\rho_k=\lambda\sigma^2/\sigma_k^2$ and $\lambda=0.23$ (DPIR default). Operators and
measurements are taken from our own implementation, so the degradations are identical
to those used by our method and the other baselines. To avoid handicapping the
baseline with a generic prior, we train the DRUNet ($32.6$M parameters) on the same
data and preprocessing as our own models, AFHQ-Cat at $256\times256$, the CelebA denoiser is trained analogously at $128\times128$. Iteration counts are tuned per
task to the plateau of a validation subset. On AFHQ-Cat, $24$ for deblurring and
super-resolution, $100$ for random inpainting, which needs many more iterations to
propagate the fill across missing pixels, and $1$ for denoising, where
$\Abm=\Ibm$ makes the data step trivial and DPIR reduces to a single denoiser pass
(more iterations are worse, $32.32$ and $32.36$\,dB at $8$ and $16$ iterations, against
$32.69$\,dB for the single pass). Same number of iterations was used for CelebA experiments as well. 
Super-resolution is initialized from a bicubic upsampling of the measurement, as in
DPIR; the remaining tasks use $\xbm_0=\Abm^\top\ybm$. We do not report box inpainting, since a Gaussian denoiser is non-generative and cannot fill a large contiguous hole. 

\paragraph{ISTA-Net$^+$.}
We use ISTA-Net$^+$~\cite{zhang2018ista}  as a representative deep-unrolling baseline. We use $K$ proximal-gradient steps, each a data-fidelity gradient step $\xbm\leftarrow\xbm-\lambda_k\Abm^\top(\Abm\xbm-\ybm)$ followed by a learned proximal (convolutional forward transform, soft-threshold, inverse transform, residual), where we trained end-to-end with the standard reconstruction plus symmetry loss ($\gamma=0.01$).  We
adapt it to three-channel images and to the measurement operators used in this paper. We use $K=9$ layers with $32$ channels and unshared weights ($0.35$M parameters), and train one model per task with Adam, learning rate $10^{-4}$, initialized at $\Abm^\top\ybm$ (bicubic for super-resolution). Box inpainting and denoising are omitted since compact convolutional proximal cannot fill a
large contiguous hole.

\section{Additional Results}\label{sec_app:add_res}
\subsection{Additional Ablation}
\paragraph{Damping coefficient $\beta$.}
Table~\ref{tab:ablation_beta} reports quality for $\beta\in\{0.25,0.5,0.75,1.0\}$ on random inpainting at $N{=}2$, $K{=}5$. The coefficient sets how much of each update comes from the learned estimator rather than the data-consistent iterate, so $\beta{=}1$ discards the data-consistent iterate from the update (though $\xbm^{k+1}$ still enters $\Rsf_\theta$ as an input) and is worst on all three metrics. Among the damped settings differences are small, all ahead of every baseline in Table~\ref{tab:repro_celeba}. Thus we use $\beta{=}0.5$ for all tasks rather than tuning per setting.

\begin{table}[t]
\centering
\caption{Damping $\beta$}
\label{tab:ablation_beta}
\begin{tabular}{cccc}
\toprule
$\beta$ & PSNR $\uparrow$ & SSIM $\uparrow$ & LPIPS $\downarrow$ \\
\midrule
0.25 & 35.30 &0.965 & 0.018 \\
0.50 & 35.03 & 0.963 &  0.018 \\
0.75 & 35.17 & 0.964 & 0.018 \\
1.00 & 34.70 & 0.961& 0.020 \\
\bottomrule
\end{tabular}
\end{table}

\subsection{Ablating the HQS}\label{app:hqs_ablation}
We evaluate whether embedding the measurement operator $\Abm$ in the generative dynamics improves upon conditioning a network on the measurement. Here, we compare this effect by using the same network, objective, training schedule, but we concatenate  $\Abm^\top\ybm$ as an additional conditioning channel with $\xbm_t$ as input in place of the HQS solver. The cost will be reduced, since we will not have   the  inner iterations of HQS. We report the results for random inpainting and deblurring tasks, across different ODE step-sizes in Table~\ref{tab:ablation_flowonly}.

\begin{table}[t]
\centering
\caption{Ablation of the operator embedding on CelebA ($100$ images). \emph{Flow-only}
is the same U-Net trained under the same objective on the concatenation of
$\Abm^\top\ybm$ and $\xbm_t$, without the HQS solver, costing one network call per
Euler step against our $K{=}5$; rows are matched by step count $N$, so our entries use
five times the function evaluations. Better of each pair in bold.}
\label{tab:ablation_flowonly}
\small
\begin{tabular}{llccccccc}
\toprule
& & \multicolumn{3}{c}{Random inpainting} & & \multicolumn{3}{c}{Deblurring} \\
\cmidrule(lr){3-5} \cmidrule(lr){7-9}
$N$ & Method & PSNR$\uparrow$ & SSIM$\uparrow$ & LPIPS$\downarrow$ & & PSNR$\uparrow$ & SSIM$\uparrow$ & LPIPS$\downarrow$ \\
\midrule
\multirow{2}{*}{$2$}
 & Flow-only & 34.54 & 0.955 & \textbf{0.0186} & & 35.77 & 0.952 & \textbf{0.0281} \\
 & Ours      & \textbf{35.03} & \textbf{0.963} & 0.0189 & & \textbf{36.04} & \textbf{0.957} & 0.0295 \\
\hdashline
\multirow{2}{*}{$4$}
 & Flow-only & 34.15 & 0.953 & 0.0172 & & 35.45 & 0.950 & \textbf{0.0256} \\
 & Ours      & \textbf{34.67} & \textbf{0.962} & \textbf{0.0169} & & \textbf{35.78} & \textbf{0.956} & 0.0280 \\
\hdashline
\multirow{2}{*}{$10$}
 & Flow-only & 33.53 & 0.948 & 0.0149 & & 34.72 & 0.943 & \textbf{0.0209} \\
 & Ours      & \textbf{33.98} & \textbf{0.956} & \textbf{0.0146} & & \textbf{35.03} & \textbf{0.950} & 0.0235 \\
\hdashline
\multirow{2}{*}{$25$}
 & Flow-only & 33.07 & 0.943 & 0.0136 & & \textbf{34.03} & \textbf{0.934} & \textbf{0.0174} \\
 & Ours      & \textbf{33.42} & \textbf{0.951} & \textbf{0.0134} & & 31.46 & 0.848 & 0.0265 \\
\bottomrule
\end{tabular}
\end{table}

At equal compute, our model at $N{=}2$ ($10$ NFEs) improves over Flow-only at $N{=}10$ ($10$ NFEs) by $1.5$\,dB PSNR on random inpainting and $1.3$\,dB on deblurring, with higher SSIM, while Flow-only attains lower LPIPS. This is consistent with the distortion--perception trade-off controlled by $N$: more Euler steps move either model toward a more perceptual operating point. At matched step count, our model improves PSNR and SSIM in all but one setting.

\subsection{Posterior recovery on a toy problem}\label{app:toy}

To evaluate posterior fidelity against ground truth, we consider a two-dimensional
problem with an analytical posterior. Let $\xbm=(u,v)\in\R^2$ follow an equally
weighted four-component Gaussian mixture with means $(\pm1.5,\pm1.5)$ and
component standard deviation $0.2$. We observe
$\ybm=\Abm\xbm+\etabm$ with $\Abm=\mathrm{diag}(1,0)$ and $\sigma=0.1$.
Thus, $u$ is observed while $p(v|\ybm)$ remains bimodal, with its mean at the
low-density point $v=0$. We replace the U-Net with a small MLP, retain
$K{=}5$, $\beta{=}0.5$, and $\rho{=}0.01$, and draw $4000$ samples for each
fixed measurement.

At $N{=}50$, the sampler recovers both posterior modes with empirical weights
$0.49/0.51$, compared with the exact $0.50/0.50$
(Figure~\ref{fig:toy_gmm}). The masked-coordinate distance is
$W_1(v)\approx0.1$, and the observed coordinate agrees with the measurement
within $\pm0.02$. The remaining error is primarily due to mild
under-dispersion of the sampled modes.

\begin{figure}[t]
\centering
\includegraphics[width=\textwidth]{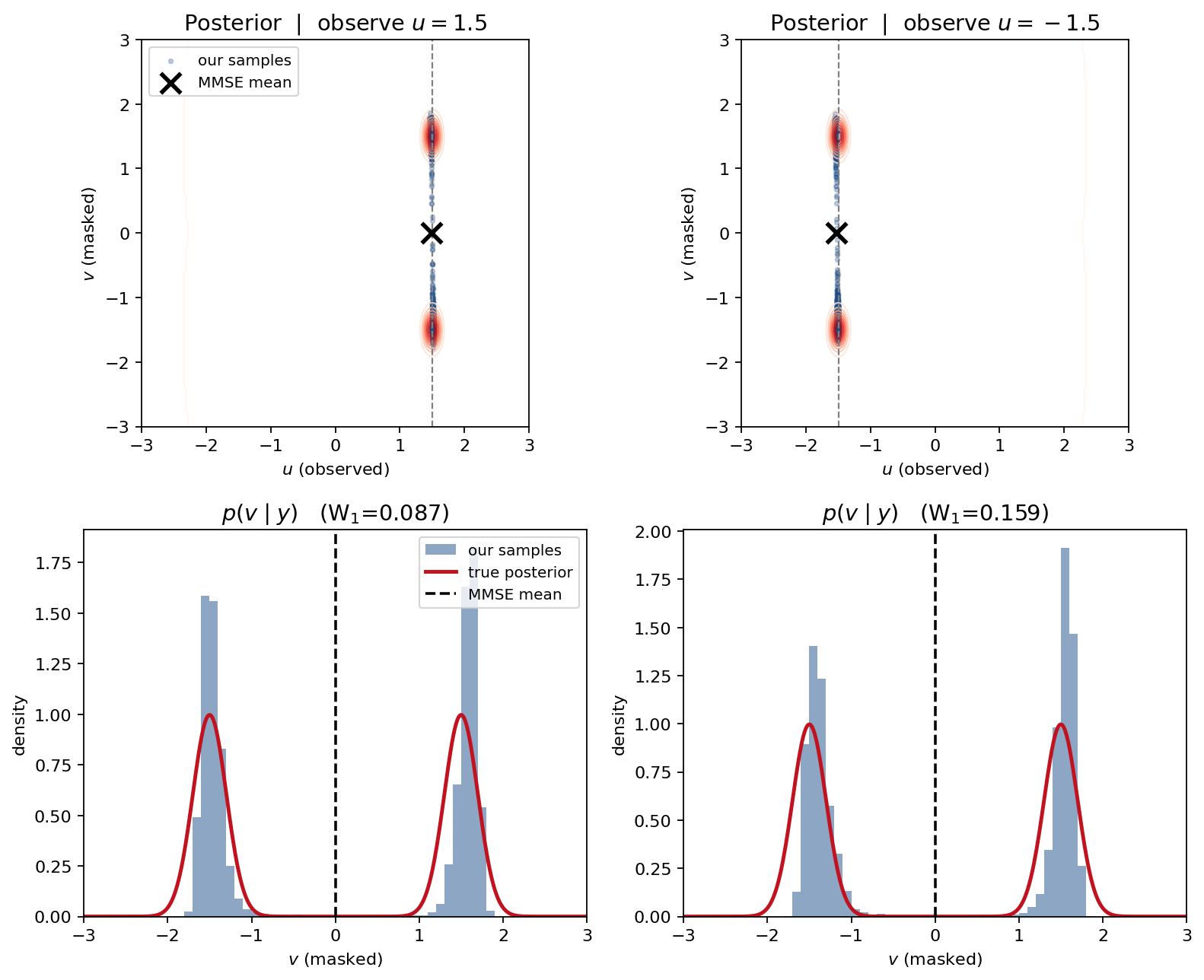}
\caption{Posterior recovery at $N{=}50$. \textbf{Top:} samples (blue) and
analytical posterior contours (red) for measurements near $u=\pm1.5$; the
posterior mean ($\times$) lies between the modes. \textbf{Bottom:} sampled
and analytical marginals $p(v|\ybm)$.}
\label{fig:toy_gmm}
\end{figure}

\paragraph{Effect of the number of sampling steps.}
Table~\ref{tab:toy_nsweep} and Figure~\ref{fig:toy_nsweep} show that the
$N{=}1$ output collapses to the posterior-mean region, whereas increasing $N$
progressively resolves both modes. Mode coverage reaches $77\%$ at $N{=}2$
and $93\%$ at $N{=}4$; for $6\le N\le25$, the recovered weights and mode
locations are within $1\%$ and $0.05$ of their analytical values.

For uniform Euler integration, $\Delta t=1/N=1-t_{N-1}$, so the final update
satisfies
\[
\xbm_N
=\E_{\thetabm}(\xbm_{N-1},t_{N-1},\ybm),
\qquad t_{N-1}=\frac{N-1}{N}.
\]
Increasing $N$ therefore moves the final estimator evaluation closer to
$t=1$, where the interpolation noise is smaller, while also changing the
state produced by the preceding Euler steps. This explains the transition
from posterior-mean-like estimates at small $N$ to mode-resolving samples at
larger $N$, consistent with the distortion--perception trade-off observed in
the image experiments.

\begin{table}[t]
\centering
\caption{Posterior recovery versus $N$ for $u{=}{+}1.5$ using $4000$
samples. The analytical posterior has negligible mass in $|v|<0.75$, equal
mode weights, and modes near $\pm1.5$.}
\label{tab:toy_nsweep}
\small
\begin{tabular}{lcccc}
\toprule
$N$ & 1 & 2 & 4 & 50 \\
\midrule
$W_1(v)\downarrow$        & 1.44 & 0.44 & 0.13 & 0.11 \\
mass in $|v|<0.75$        & 100\% & 23\% & 7\% & 0.5\% \\
mode weights              & --- & 0.50/0.50 & 0.52/0.48 & 0.49/0.51 \\
empirical modes           & 0 & $\pm1.26$ & $\pm1.45$ & $\pm1.52$ \\
\bottomrule
\end{tabular}
\end{table}

\begin{figure}[t]
\centering
\includegraphics[width=\textwidth]{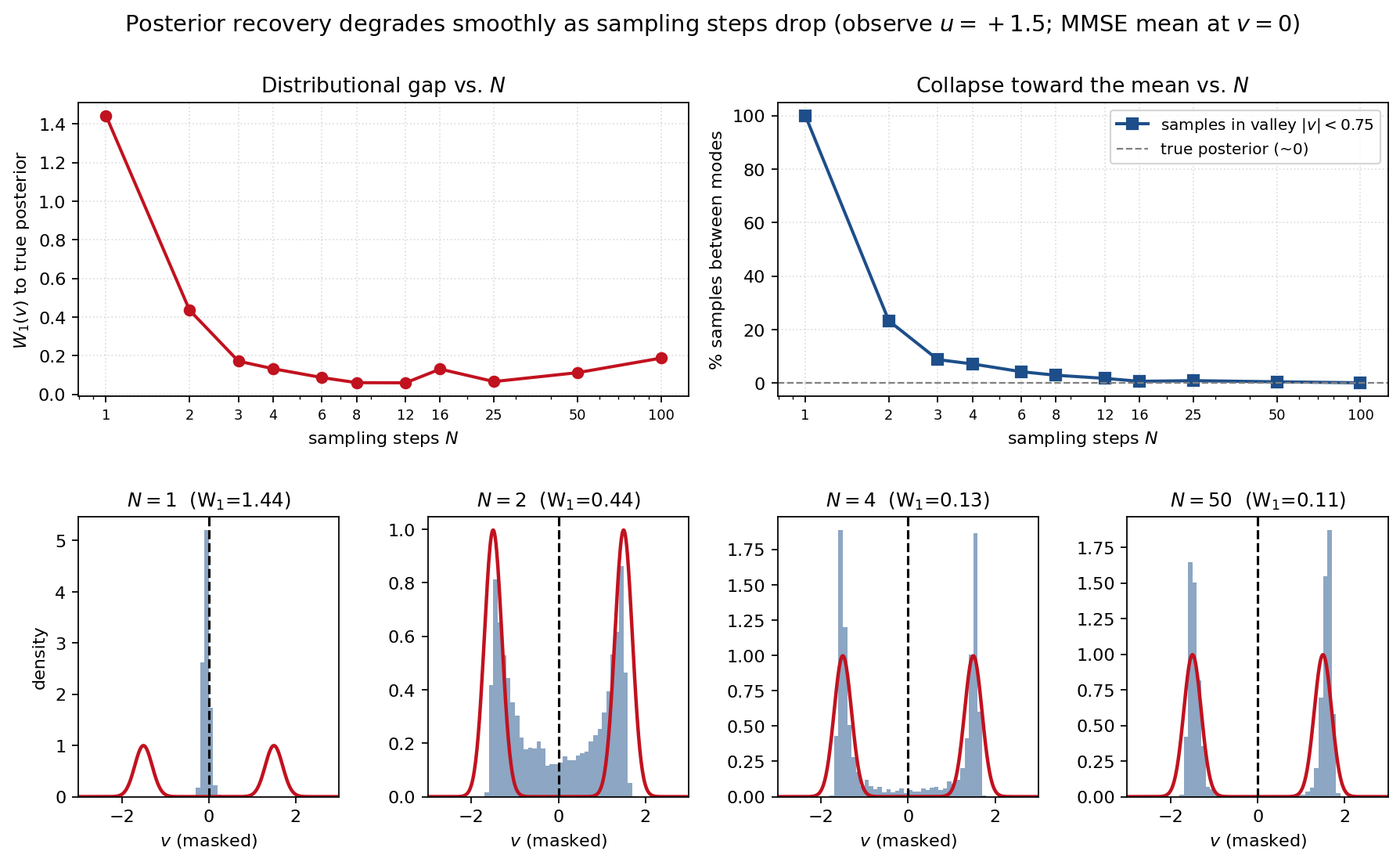}
\caption{Posterior recovery versus $N$. \textbf{Top:} $W_1(v)$ and mass in
the low-density valley. \textbf{Bottom:} sampled (blue) and analytical (red)
marginals. The modes become resolved near $N\approx6$, followed by mild drift
and under-dispersion at large $N$.}
\label{fig:toy_nsweep}
\end{figure}




\providecommand{\rotlabel}[1]{\rotatebox[origin=c]{90}{#1}}

\begin{table}[t]
\centering
\caption{Results on 1000 CelebA images. \textbf{Best} and \textbf{second} best results are shown in \bluehl{blue} and \redhl{red} .}
\label{tab:repro_celeba1000}
\setlength{\tabcolsep}{3.5pt}
\renewcommand{\arraystretch}{2.5}
\resizebox{\textwidth}{!}{
\begin{tabular}{cl *{10}{c}}
\toprule
& \multirow{2}{*}{\textbf{Metric}}
& \multirow{2}{*}{Degraded} & \multirow{2}{*}{OT-ODE}
& \multicolumn{2}{c}{PnP-Flow5}
& \multicolumn{2}{c}{Flower5-OT}
& \multicolumn{4}{c}{Ours} \\
\cmidrule(lr){5-6} \cmidrule(lr){7-8} \cmidrule(lr){9-12}
& & & & 1 & 5 & 1 & 5 & $N=2$ & $N=4$ & $N=10$ & $N=25$ \\
\midrule

\multirow{3}{*}{\rotlabel{Denoising}}
& PSNR$\uparrow$
& $20.00 \pm 0.03$ & $30.48 \pm 1.13$ & $31.79 \pm 1.21$ & $32.27 \pm 1.25$ & $32.25 \pm 1.19$ & $33.10 \pm 1.21$
& \bluehl{$33.68\pm1.18$} & \redhl{$33.25\pm1.20$} & $32.27\pm1.23$ & $31.58\pm1.22$ \\
& SSIM$\uparrow$
& $0.351 \pm 0.058$ & $0.856 \pm 0.029$ & $0.903 \pm 0.028$ & $0.909 \pm 0.028$ & $0.912 \pm 0.024$ & $0.924 \pm 0.021$
& \bluehl{$0.931\pm0.019$} & \redhl{$0.927\pm0.020$} & $0.914\pm0.023$ & $0.901\pm0.025$ \\
& LPIPS$\downarrow$
& $0.367 \pm 0.087$ & $0.033 \pm 0.014$ & $0.044 \pm 0.020$ & $0.056 \pm 0.024$ & $0.033 \pm 0.014$ & $0.038 \pm 0.016$
& $0.034\pm0.013$ & $0.034\pm0.013$ & \redhl{$0.031\pm0.012$} & \bluehl{$0.028\pm0.010$} \\
\addlinespace[3pt] \cmidrule(lr){2-12} \addlinespace[3pt]

\multirow{3}{*}{\rotlabel{Deblurring}}
& PSNR$\uparrow$
& $27.82 \pm 1.20$ & $33.00 \pm 1.68$ & $34.56 \pm 1.72$ & $34.89 \pm 1.79$ & $35.01 \pm 1.72$ & $35.71 \pm 1.80$
& \bluehl{$36.09\pm1.72$} & \redhl{$35.82\pm1.77$} & $35.05\pm1.80$ & $31.46\pm0.98$ \\
& SSIM$\uparrow$
& $0.741 \pm 0.028$ & $0.919 \pm 0.024$ & $0.935 \pm 0.016$ & $0.940 \pm 0.016$ & $0.946 \pm 0.016$ & $0.953 \pm 0.015$
& \bluehl{$0.957\pm0.015$} & \redhl{$0.955\pm0.015$} & $0.949\pm0.017$ & $0.848\pm0.039$ \\
& LPIPS$\downarrow$
& $0.123 \pm 0.035$ & $0.029 \pm 0.014$ & $0.037 \pm 0.018$ & $0.044 \pm 0.020$ & \redhl{$0.025 \pm 0.012$} & $0.031 \pm 0.014$
& $0.029\pm0.013$ & $0.027\pm0.011$ & \bluehl{$0.023\pm0.009$} & $0.026\pm0.013$ \\
\addlinespace[3pt] \cmidrule(lr){2-12} \addlinespace[3pt]

\multirow{3}{*}{\rotlabel{Super-resolution}}
& PSNR$\uparrow$
& $10.38 \pm 1.31$ & $31.43 \pm 3.50$ & $31.12 \pm 1.57$ & $31.53 \pm 1.58$ & $32.41 \pm 1.75$ & $33.14 \pm 1.80$
& \bluehl{$34.52\pm2.00$} & \redhl{$34.10\pm2.06$} & $33.27\pm2.05$ & $30.48\pm1.29$ \\
& SSIM$\uparrow$
& $0.185 \pm 0.053$ & $0.903 \pm 0.075$ & $0.900 \pm 0.030$ & $0.904 \pm 0.031$ & $0.921 \pm 0.023$ & $0.930 \pm 0.022$
& \bluehl{$0.949\pm0.017$} & \redhl{$0.946\pm0.019$} & $0.937\pm0.021$ & $0.838\pm0.040$ \\
& LPIPS$\downarrow$
& $0.829 \pm 0.063$ & \redhl{$0.026 \pm 0.063$} & $0.045 \pm 0.017$ & $0.054 \pm 0.020$ & $0.034 \pm 0.013$ & $0.039 \pm 0.015$
& $0.028\pm0.012$ & $0.027\pm0.010$ & \bluehl{$0.023\pm0.008$} & $0.028\pm0.012$ \\
\addlinespace[3pt] \cmidrule(lr){2-12} \addlinespace[3pt]

\multirow{3}{*}{\rotlabel{Random inpainting}}
& PSNR$\uparrow$
& $12.13 \pm 1.51$ & $28.87 \pm 1.89$ & $33.23 \pm 2.24$ & $34.20 \pm 2.29$ & $33.26 \pm 2.25$ & $34.15 \pm 2.29$
& \bluehl{$35.27\pm2.51$} & \redhl{$34.93\pm2.61$} & $34.23\pm2.64$ & $33.66\pm2.61$ \\
& SSIM$\uparrow$
& $0.199 \pm 0.057$ & $0.872 \pm 0.037$ & $0.944 \pm 0.019$ & $0.954 \pm 0.017$ & $0.945 \pm 0.019$ & $0.953 \pm 0.017$
& \bluehl{$0.963\pm0.016$} & \redhl{$0.962\pm0.017$} & $0.957\pm0.019$ & $0.951\pm0.022$ \\
& LPIPS$\downarrow$
& $1.027 \pm 0.077$ & $0.050 \pm 0.023$ & $0.017 \pm 0.007$ & $0.021 \pm 0.008$ & $0.017 \pm 0.007$ & $0.019 \pm 0.007$
& $0.018\pm0.009$ & $0.016\pm0.007$ & \redhl{$0.014\pm0.006$} & \bluehl{$0.013\pm0.006$} \\
\addlinespace[3pt] \cmidrule(lr){2-12} \addlinespace[3pt]

\multirow{3}{*}{\rotlabel{Box inpainting}}
& PSNR$\uparrow$
& $22.29 \pm 1.49$ & $29.68 \pm 2.56$ & $30.44 \pm 2.42$ & $30.99 \pm 2.49$ & $31.19 \pm 2.63$ & $31.81 \pm 2.69$
& \bluehl{$33.28\pm2.44$} & \redhl{$32.63\pm2.50$} & $31.95\pm2.52$ & $31.45\pm2.48$ \\
& SSIM$\uparrow$
& $0.745 \pm 0.038$ & $0.921 \pm 0.021$ & $0.934 \pm 0.017$ & $0.940 \pm 0.016$ & $0.946 \pm 0.013$ & $0.953 \pm 0.012$
& \bluehl{$0.961\pm0.012$} & \redhl{$0.959\pm0.013$} & $0.955\pm0.013$ & $0.949\pm0.014$ \\
& LPIPS$\downarrow$
& $0.211 \pm 0.035$ & $0.035 \pm 0.017$ & $0.036 \pm 0.018$ & $0.042 \pm 0.020$ & $0.021 \pm 0.014$ & $0.022 \pm 0.015$
& $0.022\pm0.012$ & $0.021\pm0.011$ & \redhl{$0.020\pm0.010$} & \bluehl{$0.019\pm0.010$} \\

\bottomrule
\end{tabular}}
\end{table}

\providecommand{\rotlabel}[1]{\rotatebox[origin=c]{90}{#1}}
\begin{table}[t]
\centering
\caption{Results on 400 AFHQ-Cat images. \textbf{Best} and \textbf{second} best results are shown in \bluehl{blue} and \redhl{red} .}
\label{tab:repro_afhq_cat400}
\setlength{\tabcolsep}{3.5pt}
\renewcommand{\arraystretch}{2.5}
\resizebox{\textwidth}{!}{
\begin{tabular}{cl *{10}{c}}
\toprule
& \multirow{2}{*}{\textbf{Metric}}
& \multirow{2}{*}{Degraded} & \multirow{2}{*}{OT-ODE}
& \multicolumn{2}{c}{PnP-Flow5}
& \multicolumn{2}{c}{Flower5-OT}
& \multicolumn{4}{c}{Ours} \\
\cmidrule(lr){5-6} \cmidrule(lr){7-8} \cmidrule(lr){9-12}
& & & & 1 & 5 & 1 & 5 & $N=2$ & $N=4$ & $N=10$ & $N=25$ \\
\midrule

\multirow{3}{*}{\rotlabel{Denoising}}
& PSNR$\uparrow$
& $20.00 \pm 0.01$ & $30.07 \pm 1.53$ & $31.22 \pm 1.60$ & $31.47 \pm 1.63$ & $31.70 \pm 1.60$ & $32.35 \pm 1.56$
& \bluehl{$32.90\pm1.55$} & \redhl{$32.50\pm1.58$} & $31.51\pm1.63$ & $29.39\pm1.15$ \\
& SSIM$\uparrow$
& $0.311 \pm 0.072$ & $0.813 \pm 0.032$ & $0.861 \pm 0.034$ & $0.862 \pm 0.035$ & $0.876 \pm 0.029$ & $0.888 \pm 0.027$
& \bluehl{$0.901\pm0.022$} & \redhl{$0.894\pm0.024$} & $0.874\pm0.028$ & $0.776\pm0.026$ \\
& LPIPS$\downarrow$
& $0.511 \pm 0.114$ & $0.079 \pm 0.021$ & $0.142 \pm 0.037$ & $0.176 \pm 0.044$ & $0.113 \pm 0.033$ & $0.127 \pm 0.036$
& $0.082\pm0.026$ & $0.080\pm0.024$ & \bluehl{$0.072\pm0.023$} & \redhl{$0.078\pm0.022$} \\
\addlinespace[3pt] \cmidrule(lr){2-12} \addlinespace[3pt]

\multirow{3}{*}{\rotlabel{Deblurring}}
& PSNR$\uparrow$
& $24.16 \pm 1.62$ & $27.23 \pm 2.31$ & $28.18 \pm 2.41$ & $28.48 \pm 2.42$ & $28.85 \pm 2.45$ & \redhl{$29.20 \pm 2.47$}
& \bluehl{$29.37\pm2.49$} & $28.88\pm2.49$ & $28.15\pm2.47$ & $27.57\pm2.44$ \\
& SSIM$\uparrow$
& $0.522 \pm 0.046$ & $0.713 \pm 0.076$ & $0.764 \pm 0.070$ & $0.769 \pm 0.070$ & $0.777 \pm 0.068$ & \redhl{$0.786 \pm 0.067$}
& \bluehl{$0.795\pm0.064$} & $0.782\pm0.067$ & $0.758\pm0.070$ & $0.735\pm0.073$ \\
& LPIPS$\downarrow$
& $0.437 \pm 0.052$ & \redhl{$0.125 \pm 0.035$} & $0.302 \pm 0.064$ & $0.331 \pm 0.067$ & $0.253 \pm 0.058$ & $0.282 \pm 0.063$
& $0.244\pm0.059$ & $0.199\pm0.051$ & $0.151\pm0.042$ & \bluehl{$0.124\pm0.035$} \\
\addlinespace[3pt] \cmidrule(lr){2-12} \addlinespace[3pt]

\multirow{3}{*}{\rotlabel{Super-resolution}}
& PSNR$\uparrow$
& $11.88 \pm 1.72$ & $26.13 \pm 2.54$ & $27.15 \pm 2.42$ & $27.61 \pm 2.43$ & $26.41 \pm 2.10$ & $26.73 \pm 2.10$
& \bluehl{$28.56\pm2.54$} & \redhl{$27.94\pm2.58$} & $27.19\pm2.56$ & $26.08\pm2.16$ \\
& SSIM$\uparrow$
& $0.219 \pm 0.073$ & $0.717 \pm 0.077$ & $0.765 \pm 0.065$ & $0.777 \pm 0.062$ & $0.743 \pm 0.069$ & $0.751 \pm 0.067$
& \bluehl{$0.803\pm0.064$} & \redhl{$0.786\pm0.069$} & $0.760\pm0.075$ & $0.665\pm0.058$ \\
& LPIPS$\downarrow$
& $0.878 \pm 0.062$ & \bluehl{$0.110 \pm 0.029$} & $0.168 \pm 0.043$ & $0.178 \pm 0.045$ & $0.275 \pm 0.051$ & $0.286 \pm 0.053$
& $0.179\pm0.042$ & $0.150\pm0.036$ & $0.122\pm0.030$ & \redhl{$0.122\pm0.026$} \\
\addlinespace[3pt] \cmidrule(lr){2-12} \addlinespace[3pt]

\multirow{3}{*}{\rotlabel{Random inpainting}}
& PSNR$\uparrow$
& $13.56 \pm 1.84$ & $29.52 \pm 2.21$ & $33.15 \pm 2.86$ & \redhl{$33.88 \pm 2.87$} & $33.12 \pm 2.84$ & $33.82 \pm 2.84$
& \bluehl{$34.12\pm2.96$} & $33.70\pm3.11$ & $32.93\pm3.16$ & $32.32\pm3.12$ \\
& SSIM$\uparrow$
& $0.230 \pm 0.070$ & $0.840 \pm 0.044$ & $0.918 \pm 0.030$ & \redhl{$0.929 \pm 0.027$} & $0.918 \pm 0.030$ & $0.927 \pm 0.027$
& \bluehl{$0.933\pm0.027$} & $0.929\pm0.031$ & $0.918\pm0.036$ & $0.907\pm0.039$ \\
& LPIPS$\downarrow$
& $1.065 \pm 0.080$ & $0.088 \pm 0.021$ & $0.040 \pm 0.012$ & $0.047 \pm 0.014$ & $0.043 \pm 0.012$ & $0.049 \pm 0.013$
& $0.047\pm0.016$ & $0.042\pm0.014$ & \redhl{$0.036\pm0.012$} & \bluehl{$0.032\pm0.012$} \\
\addlinespace[3pt] \cmidrule(lr){2-12} \addlinespace[3pt]

\multirow{3}{*}{\rotlabel{Box inpainting}}
& PSNR$\uparrow$
& $21.82 \pm 2.03$ & $24.76 \pm 2.73$ & $26.57 \pm 3.23$ & $27.19 \pm 3.26$ & $26.79 \pm 3.51$ & $27.47 \pm 3.46$
& \bluehl{$28.94\pm2.53$} & \redhl{$28.29\pm2.47$} & $27.60\pm2.42$ & $26.52\pm1.99$ \\
& SSIM$\uparrow$
& $0.743 \pm 0.038$ & $0.873 \pm 0.021$ & $0.895 \pm 0.020$ & $0.900 \pm 0.020$ & $0.915 \pm 0.014$ & \redhl{$0.923 \pm 0.014$}
& \bluehl{$0.925\pm0.012$} & $0.922\pm0.012$ & $0.916\pm0.013$ & $0.825\pm0.022$ \\
& LPIPS$\downarrow$
& $0.201 \pm 0.040$ & $0.091 \pm 0.026$ & $0.110 \pm 0.027$ & $0.127 \pm 0.029$ & $0.065 \pm 0.020$ & $0.067 \pm 0.019$
& $0.059\pm0.015$ & \redhl{$0.051\pm0.013$} & \bluehl{$0.045\pm0.012$} & $0.071\pm0.022$ \\

\bottomrule
\end{tabular}}
\end{table}

\begin{figure}[t]
\centering
\includegraphics[width=\textwidth]{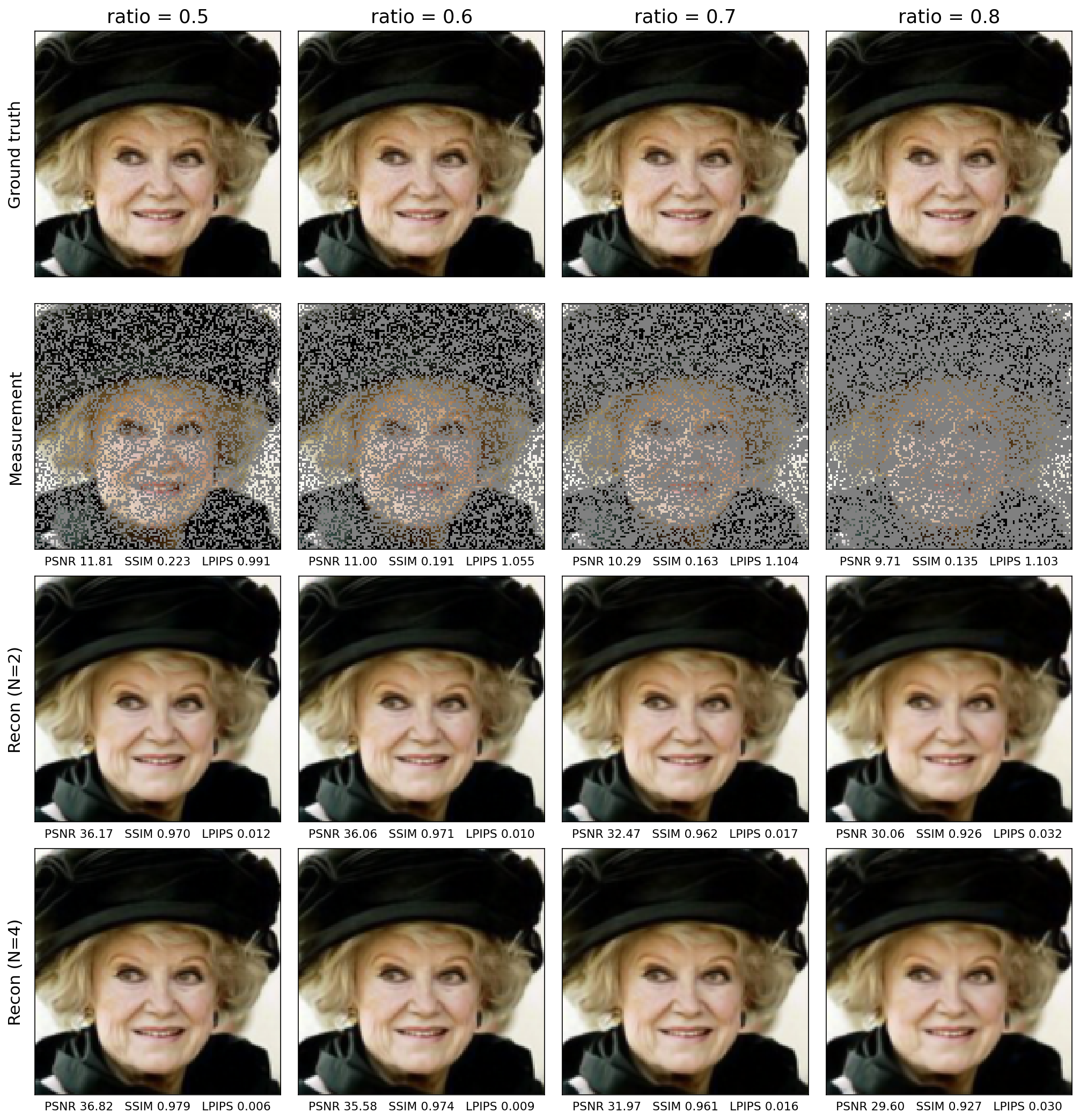}
\caption{Reconstructions under forward-model mismatch on random inpainting (CelebA,
$128 \times 128$). A single model trained at a $0.7$ mask ratio with $\sigma = 0.01$ is applied
to test measurements generated with mask ratios from $0.5$ to $0.8$. Quality is monotone in the test mask ratio, with no failure mode at either extreme. Without matched-ratio models these numbers characterize graceful degradation rather than quantifying the mismatch penalty.
Quantitative values are reported in Table~\ref{tab:ablation_operator}.}
\label{fig:mismatch}
\end{figure}

\begin{figure}[t]
\centering
\includegraphics[width=\textwidth]{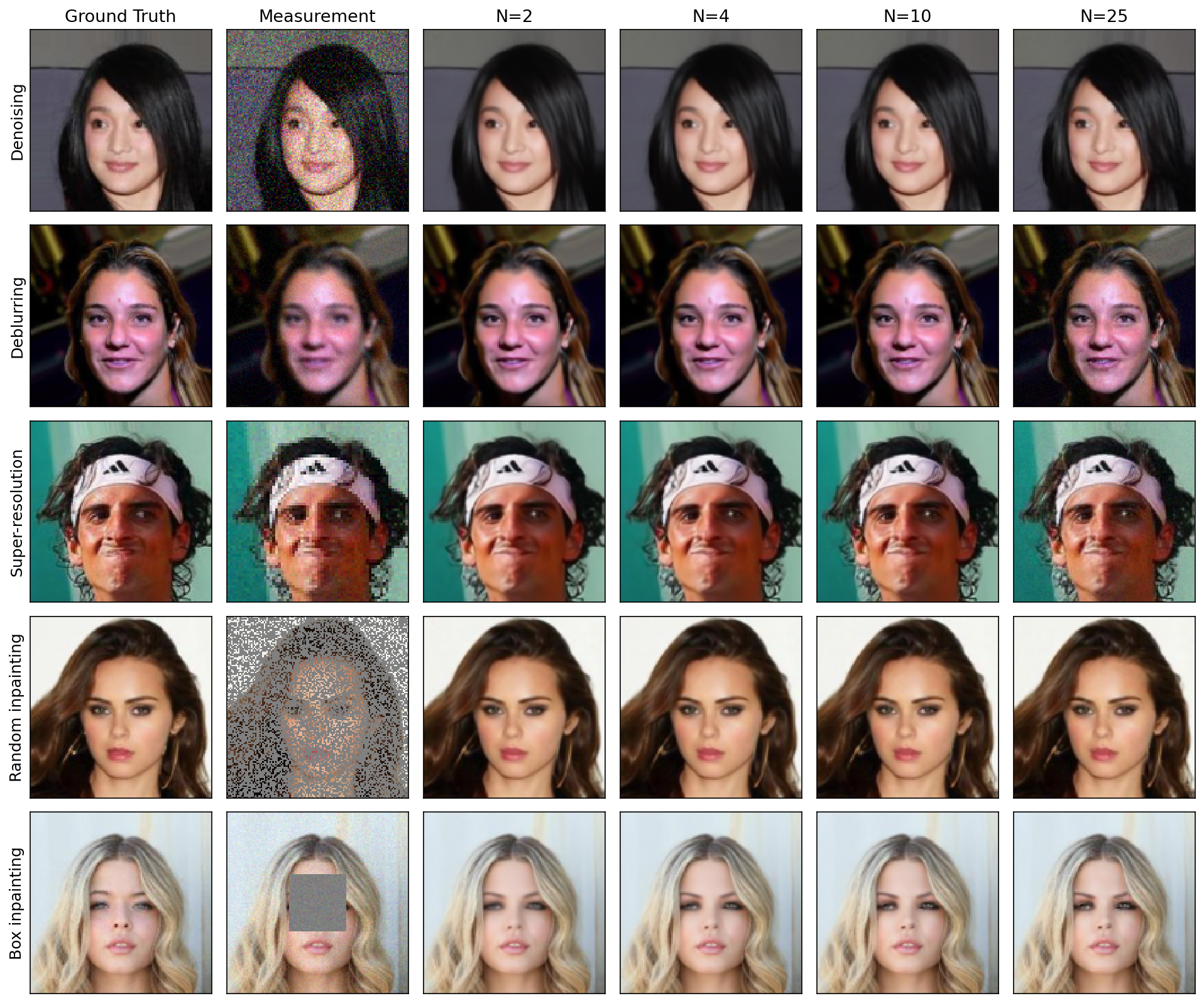}
\caption{Distortion--perception trade-off on CelebA ($128 \times 128$).}
\label{fig:pd_celeba}
\end{figure}

\begin{figure}[t]
\centering
\includegraphics[width=\textwidth]{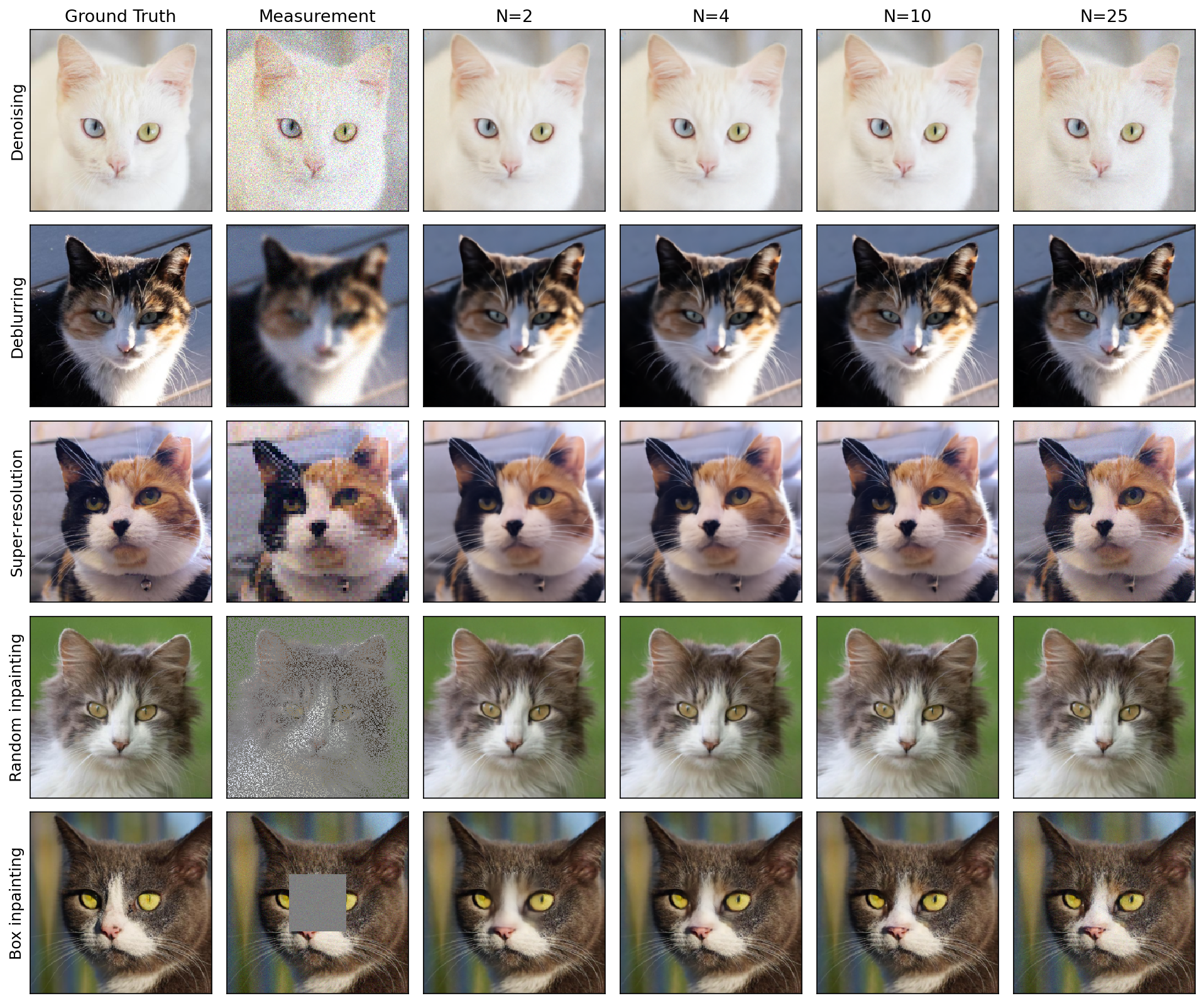}
\caption{Distortion--perception trade-off on AFHQ ($256 \times 256$).}
\label{fig:pd_afhq}
\end{figure}

\begin{figure}[t]
\centering
\includegraphics[width=0.88\textwidth]{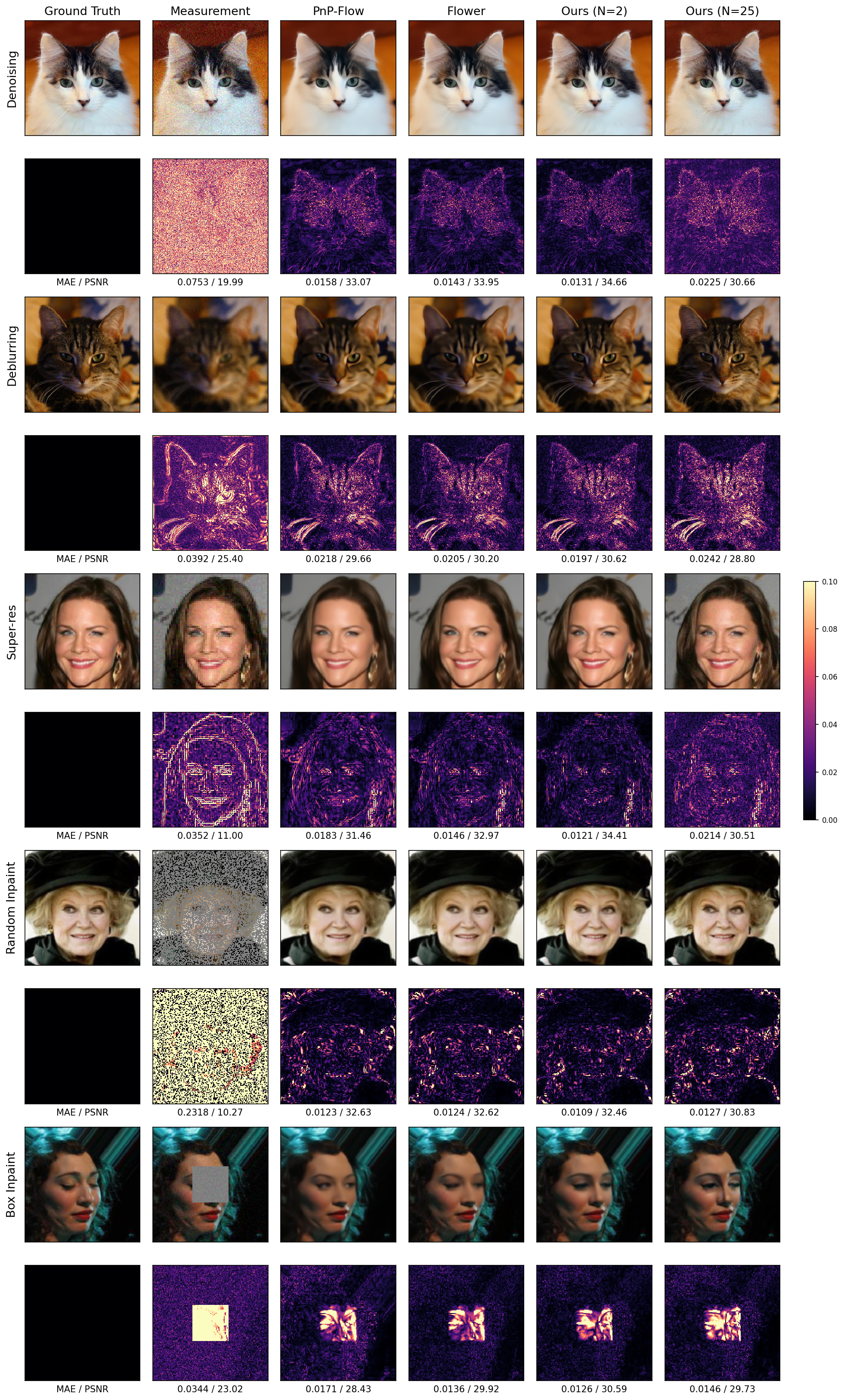}
\caption{ Reconstructions (top) and per-pixel error maps (bottom, magma, $[0, 0.1]$) for the five inverse problems (denoising, deblurring, ×2 super-resolution, random inpainting, box inpainting). Columns: ground truth, measurement, PnP-Flow, Flower, Ours (N=2), Ours (N=25). MAE / PSNR (dB) is reported under each panel. Ours (N=2) yields the lowest error on every task.}
\label{fig:all_error}
\end{figure}

\begin{figure}[t]
\centering
\includegraphics[width=0.6\textwidth]{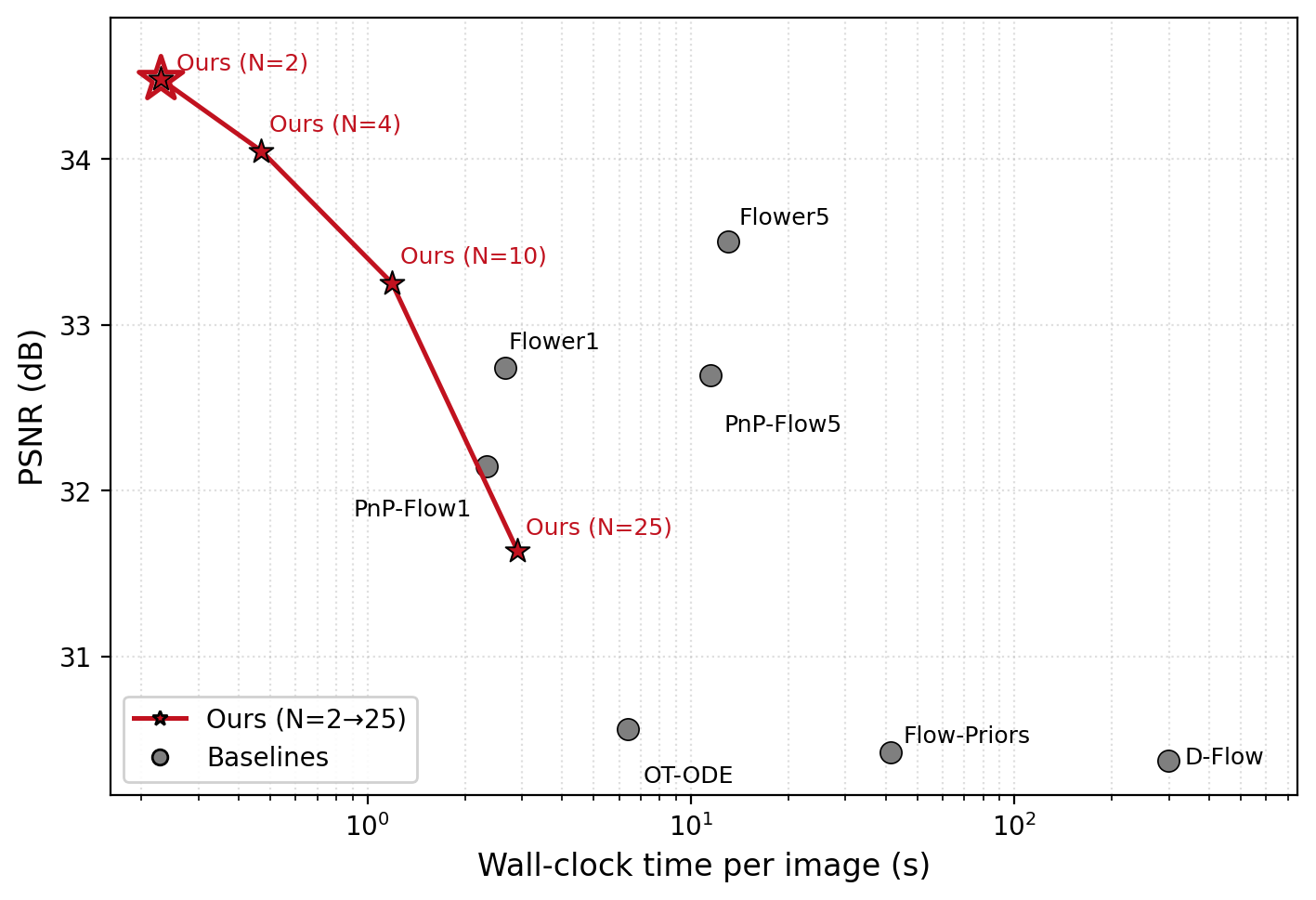}
\caption{Reconstruction quality versus inference cost on CelebA
($128\times128$). PSNR is averaged over the five restoration tasks of
Table~\ref{tab:repro_celeba} (100 test images each); the horizontal axis is per-image wall-clock, measured
on a single NVIDIA RTX A6000 and plotted on a log scale. Up and to the left is better. Our method is shown over Euler steps $N\in\{2,4,10,25\}$ with $K{=}5$ fixed, while each baseline contributes a single point. At $N{=}2$ we exceed the strongest baseline (Flower5) at $\approx50\times$ lower cost. PSNR decreases along our curve as $N$ grows because additional steps trade distortion for perceptual quality.}
\label{fig:pareto_celeba}
\end{figure}

\begin{figure}[t]
\centering
\includegraphics[width=0.9\textwidth]{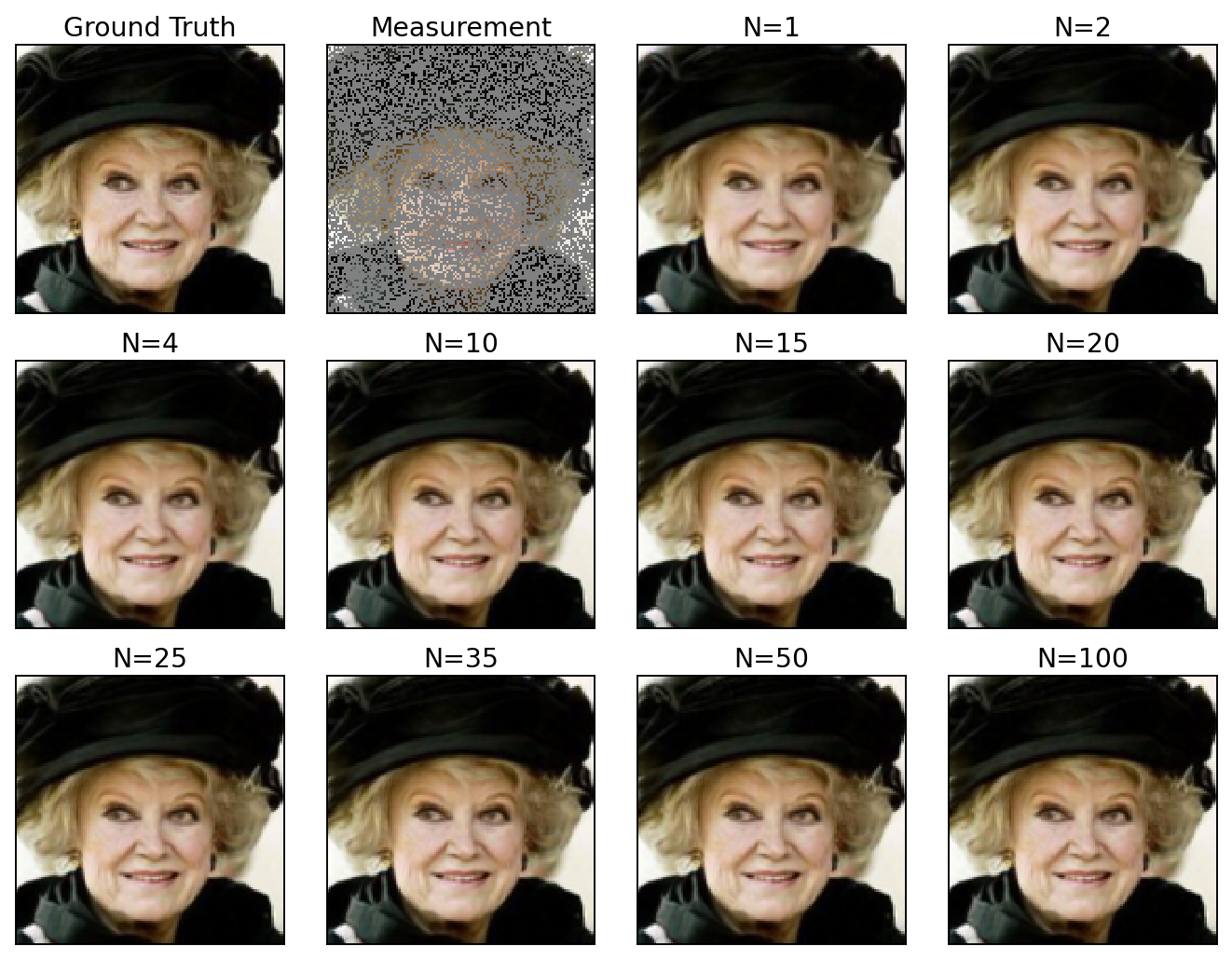}
\caption{CelebA random inpainting: reconstructions as the number of sampling steps grows ($N=1\to100$), showing added detail at higher $N$.}
\label{fig:celeba_inpaint_higher_N}
\end{figure}

\begin{figure}[t]
\centering
\includegraphics[width=0.9\textwidth]{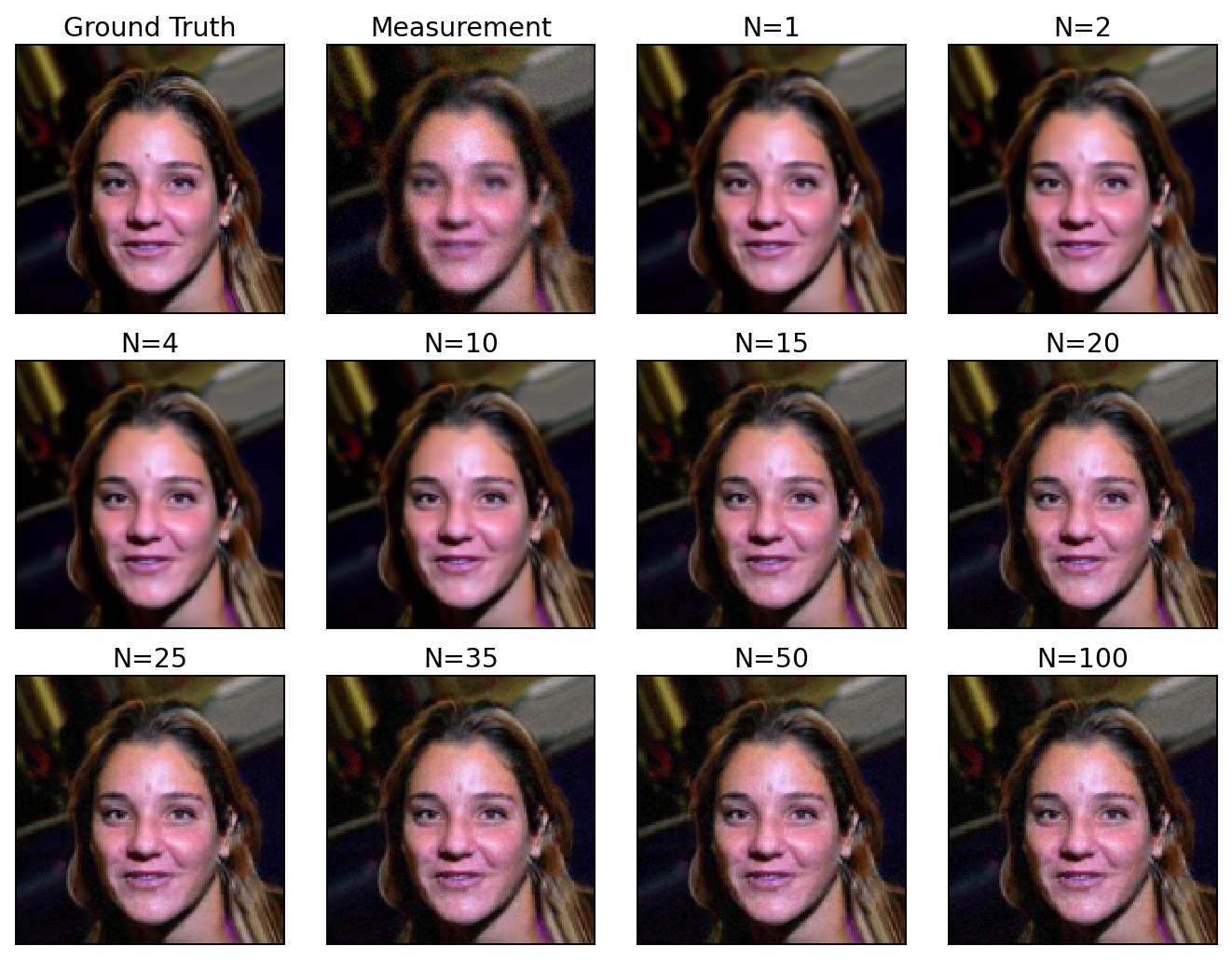}
\caption{CelebA deblurring reconstructions across sampling steps ($N=1\to100$).}
\label{fig:celeba_blur_higher_N}
\end{figure}

\begin{figure}[t]
\centering
\includegraphics[width=0.9\textwidth]{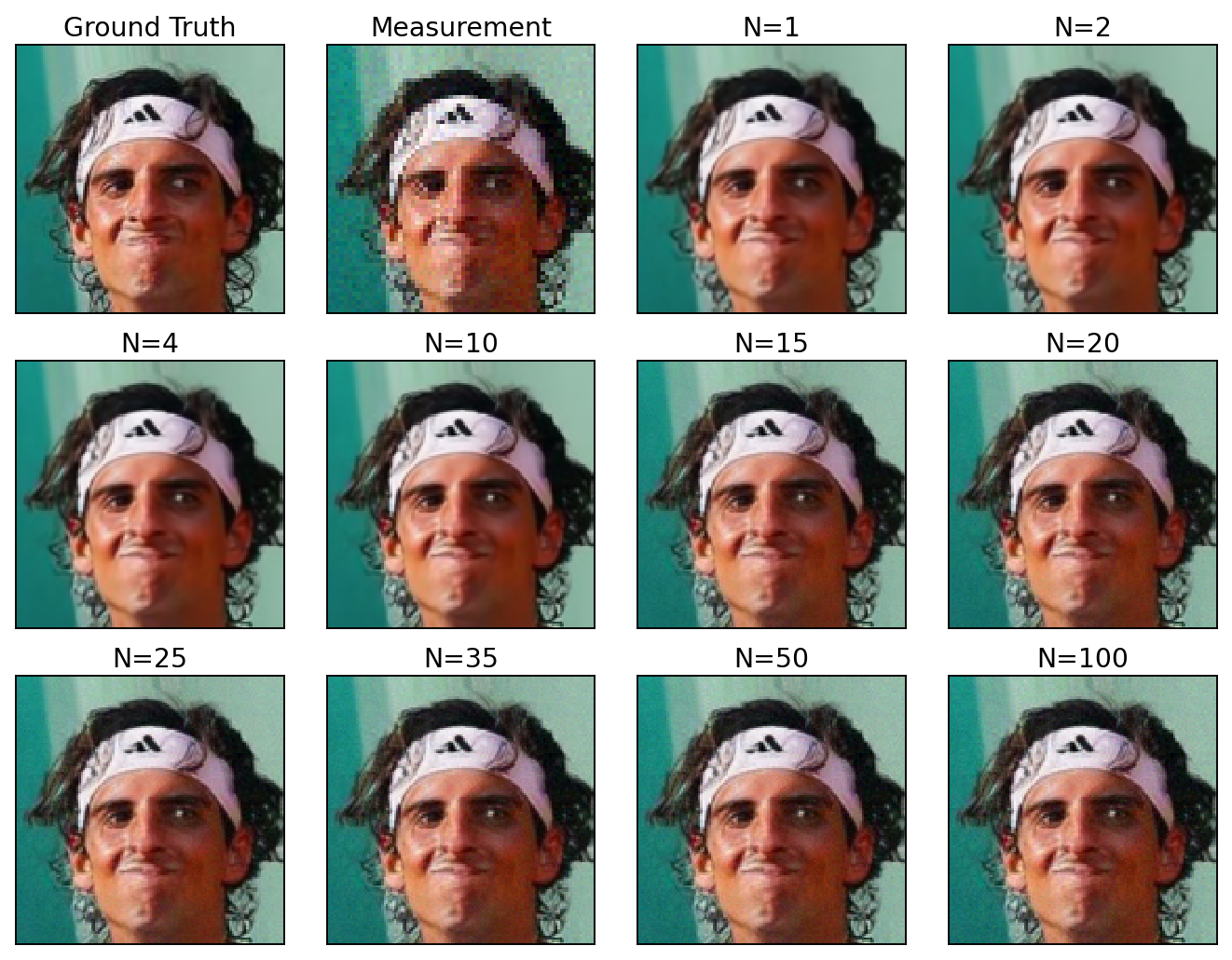}
\caption{CelebA x2 super-resolution reconstructions across sampling steps ($N=1\to100$).}
\label{fig:celeba_sr_higher_N}
\end{figure}

\begin{figure}[t]
\centering
\includegraphics[width=\textwidth]{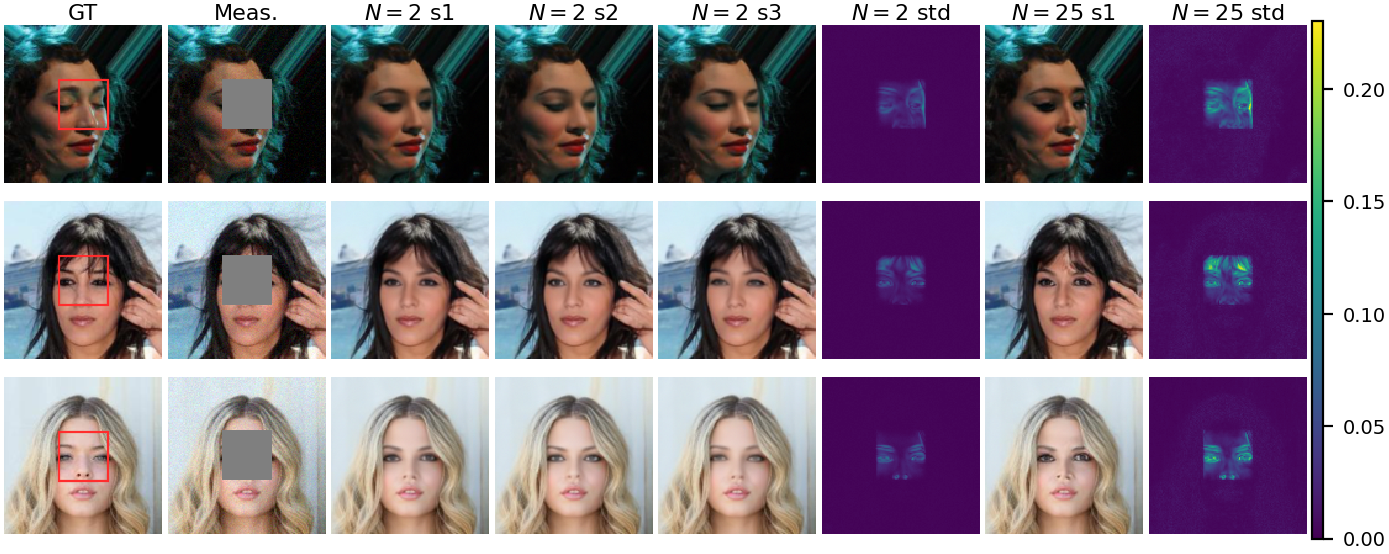}
\caption{Posterior sample diversity on box inpainting (CelebA, $128 \times 128$, $K{=}5$). For each test image the measurement $\ybm$ is held fixed and only the source draw $\xbm_0 \sim \Ncal(\bm 0,\Ibm)$ is varied. Columns show the ground truth, the measurement, three of the $S{=}8$ reconstructions at
$N{=}2$, the pixel-wise standard deviation over all eight samples, and the corresponding sample and standard deviation at $N{=}25$. Standard-deviation maps share a single color scale across the
figure. Variation is confined to the masked region and vanishes where measurements are available,
consistent with Proposition~\ref{prop:conditional_flow}, where the learned field transports the source to the
measurement-conditioned posterior rather than to a single point estimate. Longer integration
$(N{=}25)$ widens the spread, which accounts for the improvement in LPIPS and the loss in PSNR
reported in Table~\ref{tab:repro_celeba} and Table~\ref{tab:repro_celeba1000}.}
\label{fig:posterior_samples}
\end{figure}

\begin{figure}[t]
\centering
\includegraphics[width=\textwidth]{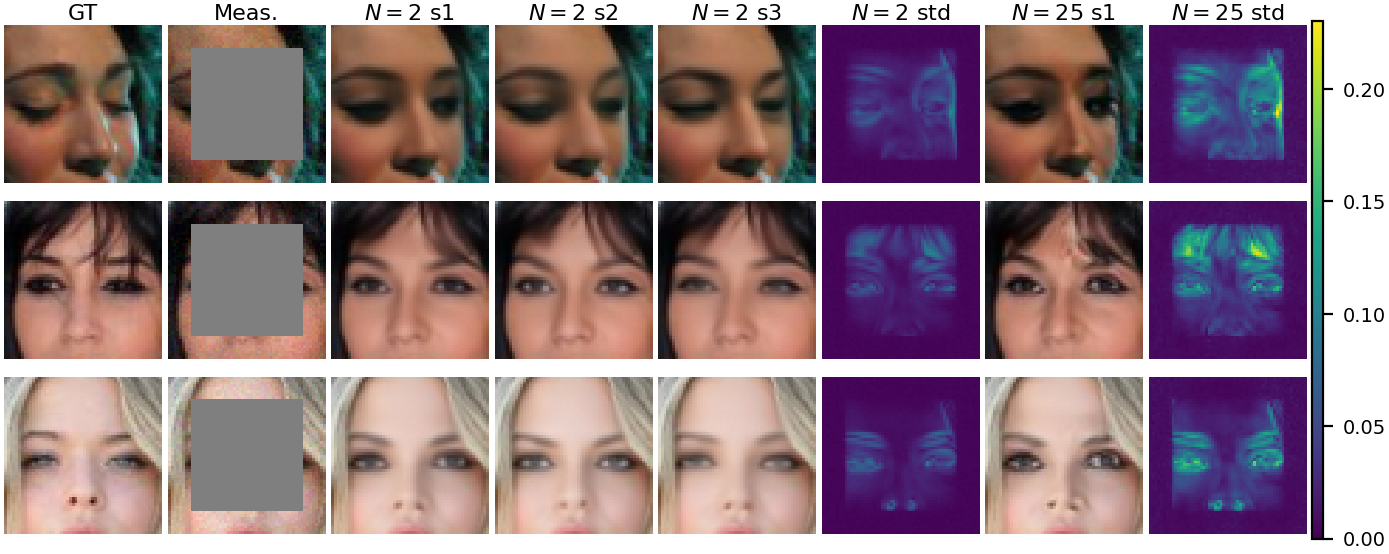}
\caption{Zoomed-in version of Figure~\ref{fig:posterior_samples}.}
\label{fig:posterior_samples_zoom}
\end{figure}

\begin{figure}[t]
\centering
\includegraphics[width=\textwidth]{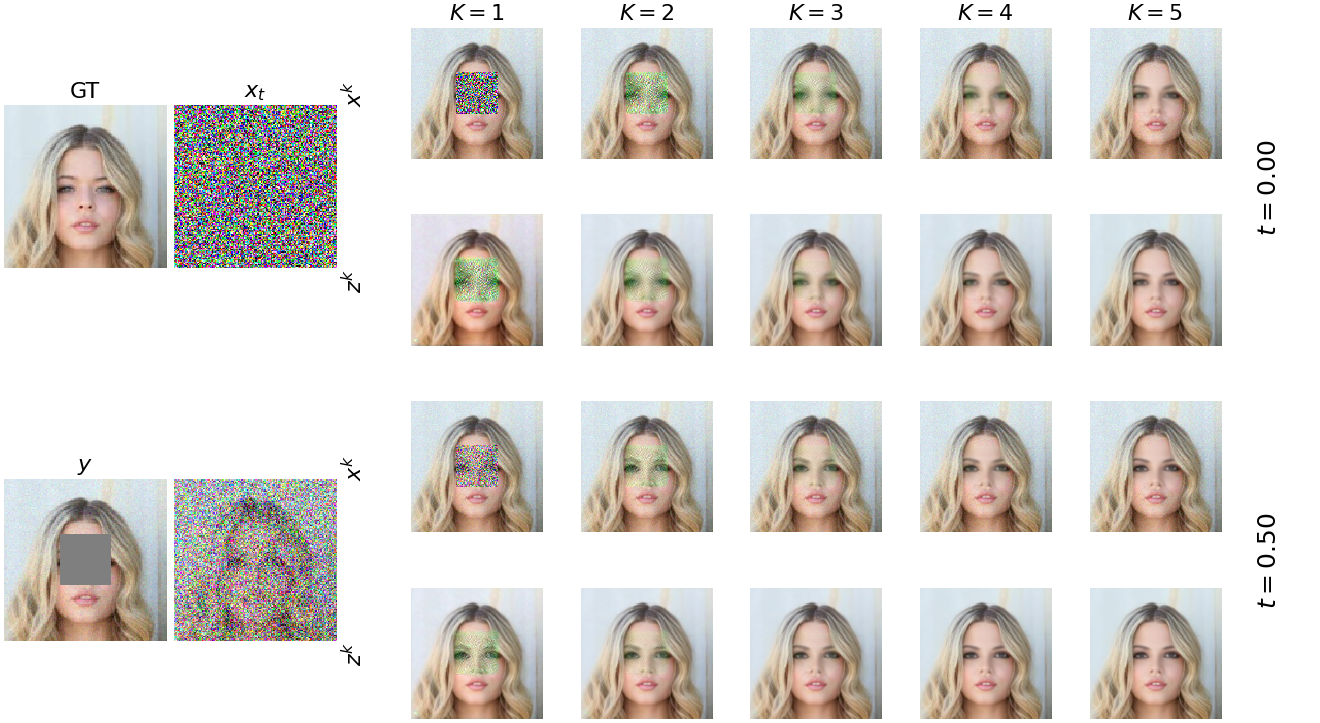}
\caption{Internal trajectory of the solver for box inpainting ($N=2$ Euler steps, $K=5$ HQS iterations each): from the ground truth and measurement, each Euler step (rows, labeled by $t$) refines the estimate through its inner iterations $K=1{:}5$, showing the data-consistency iterate $\xbm^k$ (top) and the damped learned update $\zbm^k$ (bottom).}
\label{fig:trajectory_all}
\end{figure}

\subsection{Compressed-sensing Fourier sampling.} \label{app_sec:cs}
We consider subsampled Fourier measurements $\Abm=\Pbm\Fbm$, the forward model used in accelerated MRI, where $\Fbm$ is the
two-dimensional DFT and $\Pbm$ a binary sampling mask. We evaluate
a Cartesian mask retaining $21.88\%$ of $k$-space on AFHQ-Cat, with complex Gaussian noise of standard deviation
$\sigma=0.002$, applying the operator channel by channel for color images. We
compare against Flower \cite{pourya2026flower} with $\gamma=0$, $N=100$ and
$N_{\mathrm{avg}}=1$, and against the zero-filled adjoint $\Abm^{\!*}\ybm$.
Figure~\ref{fig:cs_fourier} and Table~\ref{tab:cs} report the results.
\begin{table}[t]
    \centering 
    \caption{Compressed-sensing Fourier sampling on AFHQ-Cat (avg 100 images).}
\begin{tabular}{lccc cc}
\toprule
& \multicolumn{5}{c}{CS with Cartesian ($21.88\%$)} \\
\cmidrule(lr){2-6}
Method & Zero-filled & Flower1-OT & Ours ($N=2$) &  Ours ($N=4$) &  Ours ($N=10$)\\
\midrule
PSNR $\uparrow$ & 25.62 & 31.17 & 32.06 & 31.71 & 31.06 \\
SSIM $\uparrow$  & 0.725  & 0.870  & 0.892 & 0.885 & 0.869 \\
LPIPS $\downarrow$ &  0.343 & 0.062 & 0.074 & 0.064 &  0.054\\
\bottomrule
\end{tabular}
    \label{tab:cs}
\end{table}

\begin{figure}[t]
\centering
\includegraphics[width=\textwidth]{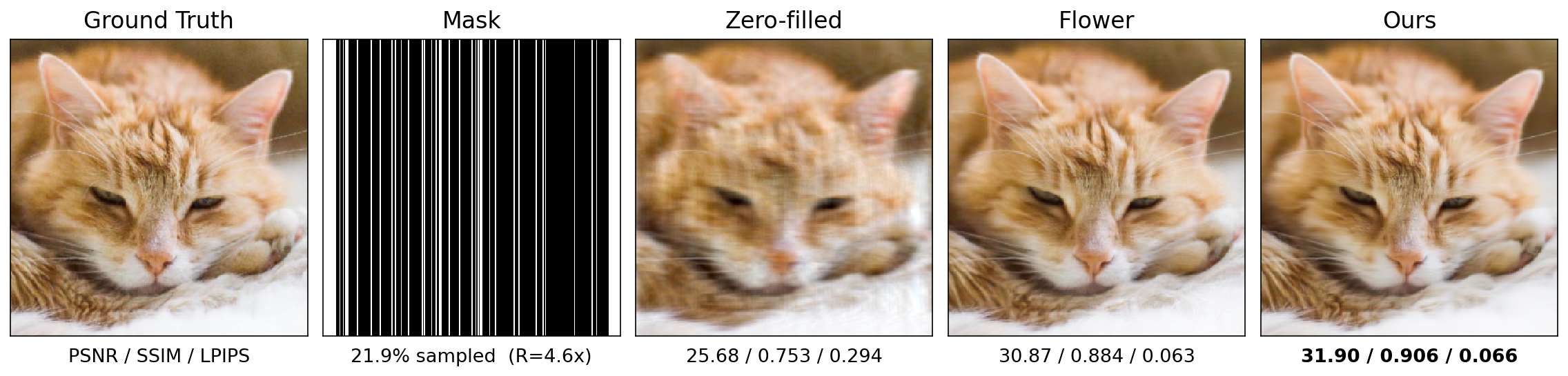}
\caption{Compressed-sensing Fourier sampling on AFHQ-Cat ($256\times256$),
$\sigma=0.002$. Columns: ground truth, zero-filled adjoint $\Abm^{\!*}\ybm$,
Flower1-OT, ours ($N{=}2$) with Cartesian mask ($21.88\%$).}
\label{fig:cs_fourier}
\end{figure}

\end{document}